\documentclass[journal]{IEEEtran}
\IEEEoverridecommandlockouts                              

\usepackage{moreverb,url}
\usepackage[colorlinks,bookmarksopen,bookmarksnumbered,citecolor=red,urlcolor=red]{hyperref}
\usepackage{algorithm,algpseudocode,float}
\usepackage{graphicx}
\usepackage{subfigure}
\usepackage{stfloats}

\usepackage{mathrsfs}
\usepackage{amsfonts}
\usepackage{psfig}
\usepackage{graphics} 
\usepackage{epsfig} 
\usepackage{times} 
\usepackage{amsmath} 
\usepackage{amssymb} 
\usepackage{multicol}
\usepackage{upgreek}
\usepackage{bm}
\usepackage{subfigure}
\usepackage{cite}
\usepackage{subeqnarray}
\usepackage{cases}
\usepackage{multirow}
\usepackage{graphics} 
\usepackage{epsfig} 
\usepackage{mathptmx} 
\usepackage{times} 
\usepackage{color}
\usepackage{booktabs}
\usepackage{float}
\usepackage{sansmath}
\usepackage{amsmath,amsfonts,amsthm,bm} 

\newcommand{\tabincell}[2]{\begin{tabular}{@{}#1@{}}#2\end{tabular}}

\makeatletter
\newenvironment{breakablealgorithm}
{
	\begin{center}
		\refstepcounter{algorithm}
		\hrule height.8pt depth0pt \kern2pt
		\renewcommand{\caption}[2][\relax]{
			{\raggedright\textbf{\ALG@name~\thealgorithm} ##2\par}%
			\ifx\relax##1\relax 
			\addcontentsline{loa}{algorithm}{\protect\numberline{\thealgorithm}##2}%
			\else 
			\addcontentsline{loa}{algorithm}{\protect\numberline{\thealgorithm}##1}%
			\fi
			\kern2pt\hrule\kern2pt
		}
	}{
		\kern2pt\hrule\relax
	\end{center}
}

\algnewcommand{\LineComment}[1]{\Statex \hskip\ALG@thistlm \(\hspace{-4mm}\triangleright\hspace{1mm}\) #1}

\algnewcommand{\AutoIndent}{\Statex \hskip\ALG@thistlm}

\algnewcommand{\Notation}{
	
	\Require{\vspace{1mm}}
	}
\algnewcommand{\Notationx}[1]{
	
	\Require{#1}
	}

\algnewcommand{\MyState}[1]{\State
	\parbox[t]{\dimexpr\linewidth-\ALG@thistlm}{\hangindent=0pt\strut\hangafter=1#1\strut}}

\algnewcommand{\MyIf}[1]{
	
	\If{
	\parbox[t]{\dimexpr\linewidth-\ALG@thistlm}{\hangindent=0pt\strut\hangafter=1#1 \textbf{then}\strut}}
	
}

\algnewcommand\algorithmicforeach{\textbf{for each:}}
\algnewcommand\ForEach{\item[ \algorithmicforeach]}
\makeatother

\newtheorem{theorem}{\bf Theorem}
\newtheorem{remark}{\it Remark}
\newtheorem{assum}{\it Assumption}

\newcommand{\BibTeX}{{\rmfamily B\kern-.05em \textsc{i\kern-.025em b}\kern-.08em
T\kern-.1667em\lower.7ex\hbox{E}\kern-.125emX}}

\ifCLASSINFOpdf
\else
\fi
\begin{document}
\title{Distributed Localization without Direct Communication Inspired by Statistical Mechanics}
\author{Jingxian Wang$^{1,\ast}$, Tianye Wang$^{1}$, Wei Wang$^{2,3,\ast}$, Xiwang Dong$^{4}$ and Yandong Wang$^{4}$
\thanks{This work was supported by the National Natural Science Foundation of China under Grants 61922008, 61973013 and 61873011.}
\thanks{$^1$School of Physics, Peking University, Beijing 100871, China.}
\thanks{$^2$Department of Urban Studies and Planning, Massachusetts Institute of Technology, Cambridge, MA 02139, USA.}
\thanks{$^3$Computer Science and Artificial Intelligence Laboratory, Massachusetts Institute of Technology, Cambridge, MA 02139, USA.}
\thanks{$^4$School of Automation Science and Electrical Engineering, Beihang University, Beijing 100191, China.}
\thanks{$^{\ast}$Corresponding authors: wangjingxian@pku.edu.cn; wweiwang@mit.edu.}
}
\maketitle

\begin{abstract}
Distributed localization is essential in many robotic collective tasks such as shape formation and self-assembly.
Inspired by the statistical mechanics of energy transition, this paper presents a fully distributed localization algorithm named as virtual particle exchange (VPE) localization algorithm, where each robot repetitively exchanges virtual particles (VPs) with neighbors and eventually obtains its relative position from the virtual particle (VP) amount it owns. Using custom-designed hardware and protocol, VPE localization algorithm allows robots to achieve localization using sensor readings only, avoiding direct communication with neighbors and keeping anonymity. Moreover, VPE localization algorithm determines the swarm center automatically, thereby eliminating the requirement of fixed beacons to embody the origin of coordinates. Theoretical analysis proves that the VPE localization algorithm can always converge to the same result regardless of initial state and has low asymptotic time and memory complexity.
Extensive localization simulations with up to 10000 robots and experiments with 52 low-cost robots are carried out, which verify that VPE localization algorithm is scalable, accurate and robust to sensor noises.
Based on the VPE localization algorithm, shape formations are further achieved in both simulations and experiments with 52 robots, illustrating that the algorithm can be directly applied to support swarm collaborative tasks.
\end{abstract}

\begin{IEEEkeywords}
Distributed Localization, Statistical Mechanics, Virtual Particle Exchange, Robot Swarm, Shape Formation
\end{IEEEkeywords}

\IEEEpeerreviewmaketitle

\section{Introduction}
\IEEEPARstart{S}{WARM} robotics has been drawing more and more attention in robotics community due to its attractive properties including scalability and adaptability, and high resistance to failure. Many collective tasks such as shape formation, reconfiguration, and self-assembly \cite{rubenstein2014programmable, zykov2005robotics,belke2017mori,kurokawa2008distributed,romanishin20153d,gilpin2008miche,formcont1,formcont2} usually require the knowledge of relative position of robots in the swarm. Deploying external localization beacons like indoor motion capture system\cite{Kushleyev2013,Preiss7989376} or outdoor Global Positioning System (GPS) \cite{Vicsek6943105, Hauert6095129} is one feasible solution for some swarm systems.
However, external beacons are not the solution for many  scenarios. For robots carrying out search-and-rescue tasks in hazardous or remote locations, setting up motion capture system is hard and GPS will be unavailable if robots are indoor or underground. Furthermore, carrying receiver of GPS can be a great burden for swarm systems consisting of numerous tiny robots which are usually weak in sensing and computation. When external beacons are not desirable, it would be more practical to rely on on-board sensing and/or internal communication among robots. The restrictions including sensing, computing, power and size  make the development of relative localization algorithms in swarm systems challenging.

Previous relative localization algorithms mainly fall into two categories: multi-hop and optimization-based. In multi-hop localization category, the simplest approach \cite{gilpin2008miche,rubenstein2014programmable,2dfaxmachine,pillai20063d} is to set up a certain amount of fixed pre-assigned beacons and let other robots acquire their positions by comparing with neighbors. This idea is further developed by adding Kalman filter \cite{savvides2002bits,Luft2018recursive} or uncertainty evaluation process \cite{rabaey2002robust,kim2018cooperative} to address the problem of error accumulation through multiple hops and achieve better localization accuracy. Some other efforts went into reducing information needed in multi-hop localization algorithms. For example, Moore \textit{et al}. employed sine theorem to estimate relative angle between robots in localization process \cite{Moore2004robust} with only relative distance information. Another example is Cornejo \textit{et al}. in which sine theorem is used to estimate relative distance between robots in localization process with only bearing measurements \cite{cornejo2013scale}. It is interesting that Pillai \textit{et al}. introduced a process to automatically select the origin of the coordinate system \cite{pillai20063d} and achieved localization without pre-defining the beacon point, though in this approach the origin of coordinate system is random, which is not favorable. A general deficiency in multi-hop methods is that one usually need to setup a fraction of predefined beacons first in order to obtain an accurate result, and low beacon fraction would lead to low localization accuracy due to propagation of error \cite{langendoen2003distributed,park2013angle}. This drawback poses difficulty to the deployment of large multi-hop localization systems.

In optimization-based category, robots obtain their localization results by implementing global optimization process like multi-dimensional scaling (MDS) \cite{shang2003localization} or regularized semi-definite programming (SDP) relaxation \cite{biswas2006semidefinite} to minimize the disagreement between the localization results and robots' measurements of relative position of neighbouring robots. However, these optimization methods are inherently centralized, thus are not suitable for large scale robot swarms. To address this problem, divide-and-conquer methods are developed to distribute optimization tasks to different robots \cite{fan2015d3d,funiak2009distributed}. However, task distribution could be complex, and alignment process \cite{funiak2009distributed} requires extra attention to avoid misaligned edges or other defects and will fail sometimes. Moreover, optimization-based methods do not always guarantee convergence to the most optimal solution and can generate highly unreasonable results in some occasions \cite{funiak2009distributed}.

It is worth noting that in many swarm tasks such as shape formation \cite{rubenstein2014programmable}, collective robotic construction \cite{Peterseneaau8479} and collective transport\cite{ZijianWangIJRR2016, ZijianWang7487163}, it is desirable to have a predictable and stable origin of coordinate system so that when robots move, their localization results will not vary significantly. This is not a problem when external or internal beacons are presented because beacons effectively define a stable coordinate system, however, things are different in homogeneous robot swarms where there is no predefined beacon. Multi-hop based algorithms rely heavily on predefined beacons. In a homogeneous robot swarm where no predefined beacon is available, the only way for multi-hop based algorithms at present is to randomly select a robot as the beacon as well as the origin \cite{pillai20063d}. Therefore, there is no guarantee that the origin will be the same in two trials of localization. Optimization-based algorithms suffer from the same issue as well because these algorithms' optimization goals usually only include the \textit{difference} of localization results, but not the absolute value of them.


When choosing an localization algorithm for a swarm robot system, one needs to evaluate multiple aspects of the algorithm. Since swarm robots typically have limited computing resources, algorithms which are stable and have low time and memory complexity are desirable. 
Moreover, hardware requirements to carry out the algorithm are equally important. It is best if the hardware requirements are as simple as possible, because complex hardware increases the failure probability in large-scale systems.
For better comparison, representative papers are analyzed in the aspect of asymptotic time and memory complexity and convergence behavior in Table \ref{tab:compinh} and in the aspect of their requirements for hardware in Table \ref{tab:compreq}.

\begin{table}[htbp]
	\begin{center}
		\begin{tabular}{c|ccccc}
		\hline \hline
		Algorithm                         & Method       & \tabincell{c}{Distri-\\buted?} & Time           & Memory     & \tabincell{c}{Local\\minima} \\ \hline
		\cite{rubenstein2014programmable} & Multi-hop    & \checkmark   & $O(l)$         & $O(1)$     &              \\
		\cite{gilpin2008miche}            & Multi-hop    & \checkmark   & $O(l)$         & $O(1)$     &              \\
		\cite{2dfaxmachine,pillai20063d}  & Multi-hop    & \checkmark   & $O(l)$         & $O(1)$     &              \\
		\cite{savvides2002bits}           & Multi-hop    & \checkmark   & $O(l)$         & $O(1)$     &              \\
		\cite{Luft2018recursive}          & Multi-hop    & \checkmark   & $O(l)$         & $O(1)$     &              \\
		\cite{rabaey2002robust}           & Multi-hop    & \checkmark   & $O(l^2)$       & $O(N)$     &              \\
		\cite{kim2018cooperative}         & Multi-hop    & \checkmark   & $O(l)$         & $O(1)$     &              \\
		\cite{Moore2004robust}            & Multi-hop    & \checkmark   & $O(N)$         & $O(1)$     &              \\
		\cite{cornejo2013scale}           & Multi-hop    & \checkmark   & $O(l)$         & $O(1)$     &              \\
		\cite{shang2003localization}      & Optimization & $\circ$      & $O(N^3)^*$     & $O(N^2)^*$ & none         \\
		\cite{biswas2006semidefinite}     & Optimization & $\circ$      & $O(N^6)$       & $O(N^4)$   & possible     \\
		\cite{fan2015d3d}                 & Optimization & $\circ$      & $O(kN^2)$      & $O(N/k)$     & possible     \\
		\cite{funiak2009distributed}      & Optimization & $\circ$      & $\approx O(N)$ & $O(1)$     & unlikely     \\ \hline \hline
		\end{tabular}
		\caption{ }
		\vspace{-4mm}Comparison of various algorithms' inherit property. While evaluating asymptotic time complexity of algorithms, let $N$ represent the number of robots in the swarm, $l$ represent the 1D span of the swarm (e.g. $l\approx\sqrt{N}$ in 2D scenarios), and let $k$ represent the number of beacons in the swarm. \checkmark means the algorithm is distributed, and $\circ$ means the algorithm can work in with both centralized and distributed theme. The problem of local minima only exists in optimization-based algorithms which used gradient descent. Asymptotic time and memory complexity of algorithm presented in \cite{shang2003localization} are calculated based on the centralized themes, and asymptotic time complexity for algorithm presented in \cite{funiak2009distributed} is estimated using the simulation results provided in the paper. 
		\vspace{-3mm}
	\end{center}
	\label{tab:compinh}
	\end{table}

	\begin{table}[htbp]
	\begin{center}
	\begin{tabular}{c|ccccc}
		\hline \hline
		Algorithm                             & Topology\hspace{-1mm}   & \hspace{-1mm}Direction\hspace{-1mm}  & \hspace{-1mm}Distance\hspace{-1mm}   & \hspace{-1mm}Beacons\hspace{-1mm}    & \hspace{-1mm}\tabincell{c}{Direct\\Comm.} \\ \hline
		\cite{rubenstein2014programmable}$^*$ & \checkmark & $ \times $ & $ \times $ & multiple   & \checkmark   \\
		\cite{gilpin2008miche}$^*$            & \checkmark & \checkmark & $ \times $ & single     & \checkmark   \\
		\cite{2dfaxmachine,pillai20063d}$^*$  & \checkmark & \checkmark & $ \times $ & single     & \checkmark   \\
		\cite{savvides2002bits}               & \checkmark & $ \times $ & \checkmark & multiple   & \checkmark   \\
		\cite{Luft2018recursive}              & \checkmark & $\circ$    & \checkmark & $\circ$    & \checkmark   \\
		\cite{rabaey2002robust}               & \checkmark & $ \times $ & \checkmark & multiple   & \checkmark   \\
		\cite{kim2018cooperative}             & \checkmark & $ \times $ & \checkmark & multiple   & \checkmark   \\
		\cite{Moore2004robust}                & \checkmark & $ \times $ & \checkmark & multiple   & \checkmark   \\
		\cite{cornejo2013scale}               & \checkmark & \checkmark & $ \times $ & $ \times $ & \checkmark   \\
		\cite{shang2003localization}          & \checkmark & $ \times $ & $\circ$    & $ \times $ & \checkmark   \\
		\cite{biswas2006semidefinite}         & \checkmark & $ \times $ & \checkmark & multiple   & \checkmark   \\
		\cite{fan2015d3d}                     & \checkmark & $ \times $ & \checkmark & multiple   & \checkmark   \\
		\cite{funiak2009distributed}          & \checkmark & \checkmark & $ \times $ & $ \times $ & \checkmark   \\ \hline \hline
		\end{tabular}
		\caption{ }
		\vspace{-4mm}Comparison of various algorithms' requirements for hardware. \checkmark means the algorithm requires the corresponding information in order to work, $\circ$ means the algorithm can work with or without this information, and $\times$ means the algorithm can not accept this information. Algorithms presented in \cite{rubenstein2014programmable,gilpin2008miche,2dfaxmachine,pillai20063d} require robots to form certain lattice structure with fixed distance between adjacent robots. 
		\vspace{-9mm}
	\end{center}
	\label{tab:compreq}
	\end{table}

It could be seen that multi-hop algorithms are more demanding in beacons, but have better asymptotic time and memory complexity, while optimization based ones are just the opposite. Few algorithm combines low asymptotic time and memory complexity, low requirement for hardware and the convenience of beacon-less. Furthermore, to the best knowledge of the authors, current algorithms in both categories all require explicit communication between robots, which means robots need to have means to identify neighboring robots and send message to and/or receive messages from neighboring robots. This demands higher hardware capability and more complex code compared with algorithms which do not require explicit communication between robots.


Motivated by the challenges stated above, a novel approach to realize fully distributed localization based on the statistical mechanics of energy transition is proposed in this paper. The localization algorithm is inspired by observations and numerical simulations of microscopic particle system. It is known that in a system of microscopic particles, regardless of initial state, the particle distribution on all energy levels would eventually reach statistical equilibrium through constant transitions and the amount of particles on a state will reflect the energy of the state. Similarly, in VPE localization algorithm, each robot represents a virtual energy level(VEL), has a numerical value representing VP amount it owns, repetitively sends VPs to neighbors and receives VPs from neighbors, and finally acquires the localization result through VP amount it owns in the final state. 

In VPE localization algorithm, because each robot will represent a VEL and exchange VPs with other robots with the same rules, it is inherently fully distributed and can work with homogeneous swarm. Because each robot only needs to remember how much VPs it owns at present, memory requirement is minimal. The algorithm can function well without beacons by design, since the particle exchange process is only related to the relative position of robots and does not need to follow any specific order. However it can also work with existing beacons to increase localization accuracy (see explanation in Section \ref{sec:basealg}). Because of a meticulous mathematical coincidence, VPE localization algorithm can be carried out without each robot knowing relative position of neighboring robots and even without direct communication between robots.

In particular, this paper contributes in the following directions:

\begin{itemize}

\item VPE localization algorithm is fully distributed, and all robots carrying out VPE localization algorithm are indistinguishable, anonymous and work in a non-communication way. Robots do not need to distinguish and communicate directly with neighbors to exchange VPs, alternatively, they can also use simple light sensors to sense and calculate VPs. By contrast, almost all previous works assumed that robots can directly communicate to exchange arbitrary data with other robots and can distinguish other robots. Therefore, our algorithm is less demanding for hardware and is more resistant to addition or removal of robots. In addition, this work also has theoretical value by presenting an example of generating specific swarm behaviors (localization and shape formation) without direct communication between individuals.

\item VPE localization algorithm is stable and has low asymptotic time and memory complexity. Each robot would do the same amount of calculation in each iteration regardless of the swarm size, which is suitable for large scale swarms. Theoretical derivation proves that the VPE localization algorithm always converges and the relationship of the localization result with the initial state only includes a global translational factor determined by initial total VP amount. Furthermore, in a simplified case where robots are positioned on a rectangular grid and can communicate with immediate neighbors, asymptotic time complexity of VPE localization algorithm is proved to be $O(l)$ where $l$ is the 1D dimension of the swarm. Among all representative algorithms in Table \ref{tab:compinh}, only pure multi-hop methods have asymptotic time complexity of $O(l)$.
Moreover, local minimal problem \cite{funiak2009distributed} and complicated alignment process in divide-and-conquer methods will not appear in VPE localization algorithm.

\item The algorithm has no requirements for internal or external beacons to embody the origin of coordinate system. Instead, in VPE localization algorithm, the origin of coordinate system will automatically locate near the center of the swarm. This feature means that all robots in the localization can be homogeneous, which simplifies the deployment of robot swarms. To the best of the authors' knowledge, this feature is not presented in all previous works concerning distributed localization.

\item The algorithm can be executed with minimal hardware resources and is flexible as well. In Section \ref{sec:VPEcode} we present a way to execute VPE localization algorithm without explicitly detecting neighboring robots' position or bearing or even without identifying neighbours. However in Section \ref{sec:basealg} we also show that VPE localization algorithm can work with robots with communication ability and can utilize beacons as well as relative direction and distance information between robots to generate more accurate localization result.

\item Localization and shape formation experiments are carried out on 52 self-designed low-cost robots, proving that the proposed localization algorithm is suitable for real world applications. 
\end{itemize}

The paper is organized as follows. Section \ref{sec:VPEmethod} presents general VPE localization algorithm as well as a modified version which does not require direct communication between robots along with detailed pseudo-code and hardware requirements. Section \ref{sec:anaVPEmethod} demonstrates theoretical derivation about the convergence and asymptotic time complexity of proposed localization algorithm. To demonstrate the performance of VPE localization algorithm in simulations and on robot swarms, Sections \ref{sec:locresult} and \ref{sec:sfresult} present results in localization and shape formation simulations and experiments. Section \ref{sec:discussion} addresses several issues regarding performance of VPE localization algorithm and hardware setups. Section \ref{sec:conclusion} concludes the work. Additionally, we also recorded a video for better illustration\footnote{Video can be found at \url{https://www.youtube.com/watch?v=XxEpmcOvr18} or \url{https://www.bilibili.com/video/av93915568/}.} and open-sourced our code for simulation\footnote{Code written in Mathematica could be downloaded at \url{https://github.com/wjxway/VPE-localization-algorithm}}.

Throughout this paper, for simplicity of notation, $\bm{0}_n$ and $\bm{1}_n$ will denote zero column vectors and one column vectors with dimension $n$ respectively, $\bm{0}_{m\times n}$ will denote zero matrices with $m$ rows and $n$ columns, $\bm{I}_n$ will denote identity matrices of rank $n$, $\hat{\bm{x}}$ and $\hat{\bm{r}}$ will denote the unit vector along $x$ axis and vector $\vec{\bm{r}}$ respectively, $\bm{U}^\mathrm{T}$ will denote the transpose of vector or matrix $\bm{U}$, $x\propto y$ means $x$ is proportional to $y$, $\lfloor x\rfloor$ will denote the largest integer smaller or equal to $x$, $\vec{\bm{u}}\cdot\vec{\bm{v}}$ will denote the scalar product of vector $\vec{\bm{u}}\text{ and }\vec{\bm{v}}$, $\|\bm{v}\|$ will denote $\sqrt{\bm{v}^\mathrm{T}\bm{v}}$, and $\mathop{\mathrm{diag}}\{a_1,a_2,\cdots,a_l\}$ will denote a diagonal matrix of dimension $l\times l$ with $a_1, \cdots, a_l$ as its diagonal element.

\section{Design of VPE localization algorithm}\label{sec:VPEmethod}
In this section, we first explain the physical principle of the VPE method, and then propose a general algorithm of distributed localization using VPE method.
\subsection{Physical basis of VPE localization algorithm}\label{sec:phy}
In the most probable state of a closed system with constant temperature $T$ and is composed of identical near-independent traditional particles (i.e. Nitrogen molecules in a room), amount of particles on different energy levels follows the Boltzmann distribution

\vspace{-2mm}
\begin{equation}
	n_i \propto e^{-\beta E_i}
	\label{eqn:1}
\end{equation}
\vspace{-5mm}


\noindent where $n_i$ is the number of particles on a specific state $i$ with energy $E_i$, and $\beta=\frac{1}{\mathrm{k_B}T}$ where $\mathrm{k_B}$ is Boltzmann constant. This equation shows that particles tend to stay in states with lower energy. Describing it in the opposite way, it can also be said that the amount of particles in the state reflects the energy of the state.
In order to observe the evolution of such systems, Markov chain Monte Carlo (MCMC) simulations are often used \cite{Berg2004Markov}, in which a large number of particles are spawned and allowed to transit from a state to another with a certain probability. Eventually, when equilibrium is reached, the distribution of particles will follow a certain distribution which, in this case, is the Boltzmann distribution. It is obvious that a transition probability $P$ in the form of

\vspace{-2mm}
\begin{equation}
	P \propto e^{-\frac{\beta \Delta E}{2}}
	\label{eqn:2}
\end{equation}
\vspace{-5mm}

\noindent can lead to Boltzmann distribution, where $\Delta E$ denotes the energy required in a particle's transition, in other words, the difference in energy when a particle stays in the state after transition and before transition.

The relation between Equation \ref{eqn:1} and \ref{eqn:2} is intriguing. While the transition probability $P$ of particles depends solely on $\Delta E$ which is the \textit{relative} energy difference, when the particle distribution reaches equilibrium state, the amount of particles in a specific state will indicate \textit{global} information of the energy level of state $E_i$.

Mimicking the behavior of particles in the thermal dynamic system, we let each robot correspond to a virtual energy level (VEL) which holds a certain amount of virtual particles (VPs) and let VPs transit from a VEL to another through communications between robots with transition probability similar to (\ref{eqn:2}) but replace the energy $E_i$ with the robots' $x$ coordinate $x_i$. We can expect that in the equilibrium state, the distribution of VPs will be related to the $x$ coordinates of robots. In this way, a robot can determine its relative position in the swarm by the amount of VPs it owns in the equilibrium state. We name this as Virtual Particle Exchange (VPE) localization algorithm.

\subsection{Localization based on VPE localization algorithm}\label{sec:basealg}


When trying to perform localization using VPE localization algorithm, a orthogonal coordinate system should be constructed first, and then each components of robots' position can be determined separately by executing VPE localization algorithm. For example, in a scenario when robots are distributed on a 2D plane, each robot should first agree on a common $x$ and $y$ direction (possibly using compass), then run VPE localization algorithm twice to determine the $x$ and $y$ components of its location.

Writing down the algorithm described in Section \ref{sec:phy} in detail, we get the general process of VPE localization algorithm which consists of three main steps:

\begin{enumerate}
	\item [(1)] Each robot is given a certain amount of VPs.
	\item [(2)] Robots exchange VPs repetitively according to the amount of VPs they own and the relative displacement.
	\item [(3)] Robots extract desired results from the amount of VPs they own in the equilibrium state.
\end{enumerate}

\noindent In further discussion we name the second step as VPE process, which is the most important step in VPE localization algorithm.\\

Mathematically speaking, in VPE process the configuration of robot swarm could be described by an undirected weighed graph $G(\bm{V},\bm{\Gamma})$, where $\bm{V}$ is the set of all $l$ robots (vertexes) and $\bm{\Gamma}$ is the weighted adjacency matrix with its element $\bm{\Gamma}_{i,j}$ representing how closely robot $i$ is connected to robot $j$. $\bm{\Gamma}_{i,j}$ could be relevant to connection topology, hardware and software settings, and relative distance between robots. For example, In Section \ref{sec:VPEcode}, $\bm{\Gamma}_{i,j}$ represents the intensity of light sensed by robot $j$ when robot $i$ is emitting isotropic light of unit intensity. In this case, due to the reversibility of light, $\bm{\Gamma}_{i,j}=\bm{\Gamma}_{j,i}$ holds, thus $G$ is an undirected weighed graph.

To describe the state of the system, a series of vector $\bm{\xi}^{(n)}$ are used, where superscript $n$ stands for iteration count (which can be omitted if unnecessary). Element $i$ in vector $\bm{\xi}^{(n)}$ ($\bm{\xi}_i^{(n)}$) represents the amount of VP robot $i$ owns. Using previous definitions, the VPE localization algorithm is given in Algorithm \ref{alg:1}.

\begin{breakablealgorithm}\label{alg:1}
\caption{VPE localization algorithm}
	\begin{algorithmic}[1]
	\Require Topology of robot swarm $G(\bm{V},\bm{\Gamma})$, relative distance between robots $\vec{\bm{r}}_{i,j}=\vec{\bm{r}}_{j}-\vec{\bm{r}}_{i}$.
	\Ensure {$x$ component of localization result, $\bm{\chi}_{i,x}$, for all robots.}
	\Notation
	\Notationx{$n_{m}$ - iterations to calculate (predefined).}
	\Notationx{$\mathrm{k}_0$ - constant controlling the speed of particle transition.}
	\Notationx{$\mathrm{k}$ - constant controlling the anisotropical distribution of VP transition.}
	\ForEach{robot $i\in \bm{V}$.}
	\LineComment{\textit{Initialization:}}
	\State $\bm{\xi}_i\gets1$
	\LineComment{\textit{VPE process:}}
	\For{$n=0\to n_{m}-1$}
		 \MyState{
		 \begin{eqnarray}
		     \bm{\xi}_i\hspace{-3mm}&\gets&\hspace{-3mm}\bm{\xi}_i+\sum_{j\neq i}\bm{\xi}_j\,P_{j,i}-\bm{\xi}_i\,P_{i,j}\label{eqn:VPevo}\\
		     \textbf{s.t.}\hspace{5mm}P_{i,j}\hspace{-3mm}&=&\hspace{-3mm}\bm{\Gamma}_{i,j}\,\mathrm{k}_0\,e^{-\mathrm{k}\vec{\bm{r}}_{i,j}\cdot\hat{\bm{x}}}\label{eqn:P}
		 \end{eqnarray}
		 
		 \noindent where $P_{i,j}$ represents the possibility for VPs owned by robot $i$ to transit to robot $j$ in this iteration.\label{alg:step3}
		 
		 %
		 }
	\EndFor
	\LineComment{\textit{Result Extraction:}}
	\State $\bm{\chi}_{i,x}\gets -\frac{\ln\bm{\xi}_i}{2\mathrm{k}}$\label{alg:resext}
	\end{algorithmic}
\end{breakablealgorithm}

The similarity between VPE localization algorithm and physical particle transition is clear. In the initial state, all robots own 1 unit of VPs, and then VPs start transiting between robots (VELs) with probability $P_{i,j}=\mathrm{k_0}\bm{\Gamma}_{i,j} e^{-\mathrm{k}\vec{\bm{r}}_{i,j}\cdot\hat{\bm{x}}}$, which is (\ref{eqn:2}) with substitution $\Delta E=\frac{2\mathrm{k}}{\beta}\vec{\bm{r}}_{i,j}\cdot\hat{\bm{x}}$. Thus, the corresponding potential energy of a VP possessed by robot $i$ is $E_i=\frac{2\mathrm{k}}{\beta}x_i+\mathrm{E_0}$ where $x_i$ is the $x$ coordinate of robot $i$ and $\mathrm{E_0}$ is a constant. In the equilibrium state of such a system, as described in (\ref{eqn:1}), $\bm{\xi}_i\propto e^{-\beta E_i}\propto e^{-2\mathrm{k}x_i}$ and $x_i=-\frac{\ln{\bm{\xi}_i}}{2\mathrm{k}}+\mathrm{x}_0=\bm{\chi}_{i,x}+\mathrm{x}_0$ where $\mathrm{x_0}$ is a global constant representing the shift of the origin of the coordinate system. The equation takes the same form as the equation used in step \ref{alg:resext} in Algorithm \ref{alg:1}.

An illustration of VPE process is shown in Figure \ref{fig:illus1}, which shows that after a few iterations, VPs gather on robots with smaller $x$ coordinates and VP distribution reaches a equilibrium state in which the amount of VPs received and sent by a robot in each iteration equals. In the equilibrium state, the amount of VPs each robot owns follows the relationship $\bm{\xi}_i \propto -e^{2\mathrm{k}x_i}$ as previously described.

\begin{figure}[htbp]
\begin{center}
	\includegraphics[width=.5 \textwidth]{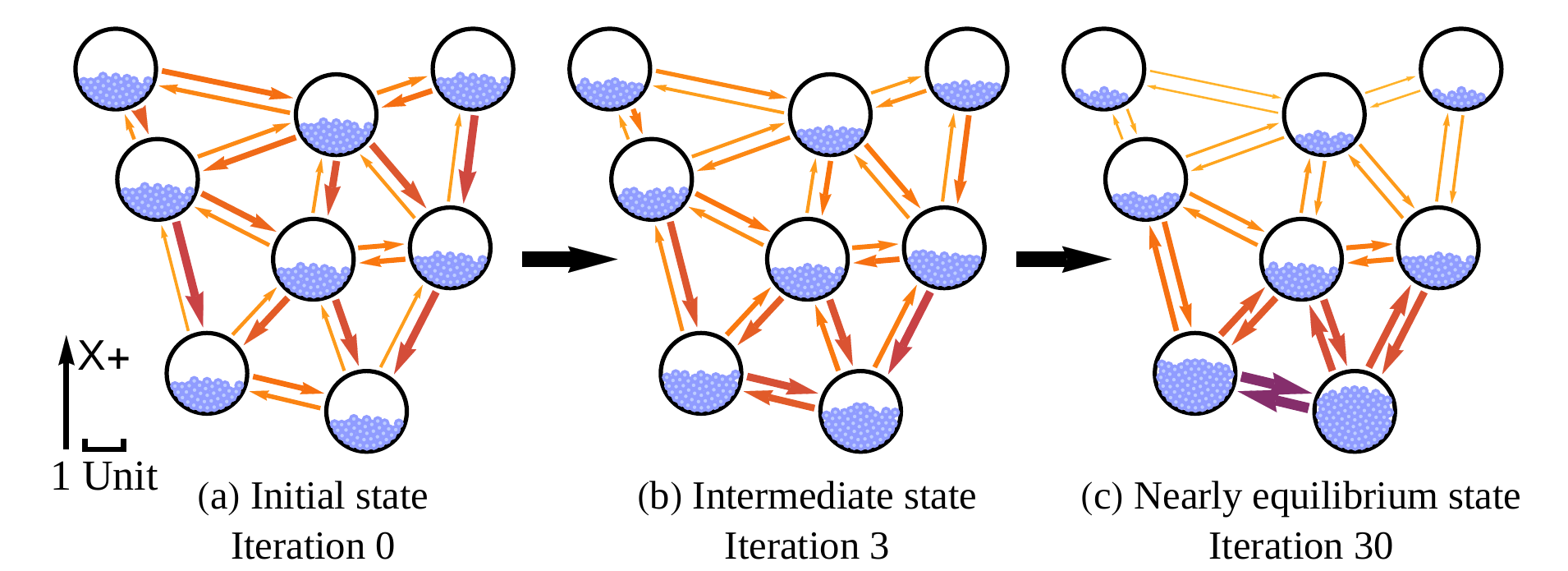}
	\vspace{-3mm}
	\caption
	{
		Illustration of different states in a VPE process. Each Black circle represents a robot unit positioned at the center of it and small blue circles inside represent VPs a robot owns. Robots can exchange VPs with robots in 6 units. The exchange of VPs is represented by arrows between robots, with the amount of VPs transferred reflected by the thickness of the arrow. Note that in real applications VP amounts are real numbers instead of integers.
	}
	\vspace{-3mm}
	\label{fig:illus1}
\end{center}
\end{figure}

In practice, in order to set the origin near the center of swarm, VPE process is run twice using opposite $x+$ direction and then the localization result is calculated using

\begin{equation}
	\bm{\chi}_{i,x}=\frac{\ln\bm{\xi}_{i,-}^{(f)}-\ln\bm{\xi}_{i,+}^{(f)}}{4\mathrm{k}}
\end{equation}

\noindent where $\bm{\xi}_{i,+}^{(f)}$ and $\bm{\xi}_{i,-}^{(f)}$ are VP amount in final states of two different VPE processes executed with opposite $x+$ direction. This method is used in all further simulations and experiments.

Algorithm \ref{alg:1} can be directly implemented in robot systems where each robot can obtain the relative position of other neighboring robots and can exchange data via one-to-one communication (configuration 1 in Table \ref{tab:compreq}), and can also work with beacons with pre-defined position by fixing the amount of VPs owned by beacon robots $i$ at $\bm{\xi}_i=\mathrm{exp}(-2\mathrm{k}\vec{\bm{r}}_{i}\cdot\hat{x})$. Though these two requirements can easily be met in modular robot swarm \cite{zykov2005robotics,belke2017mori,kurokawa2008distributed,romanishin20153d,gilpin2008miche}, for other swarms in which robots are scattered in 2D or 3D \cite{rubenstein2014programmable,Kushleyev2013,Preiss7989376,Vicsek6943105,Hauert6095129}, these two requirements can be too hard or too expensive to achieve.

Slight modification can be made to the algorithm to make it work with fewer available information and consume less hardware resource (details provided in the next subsection). In scenarios where robots are approximately evenly spaced, $\vec{\bm{r}}$ can be approximated by $\mathrm{r}_0\,\hat{\bm{r}}$ where $\mathrm{r}_0$ is the weighed average distance between connected robots. In this case, Equation \ref{eqn:P} and line \ref{alg:resext} in Algorithm \ref{alg:1} should be changed to

\vspace{-2mm}
\begin{eqnarray}
    P_{i,j}\hspace{-3mm}&=&\hspace{-3mm}\bm{\Gamma}_{i,j}\,\mathrm{k}_1\,e^{-\mathrm{k}\hat{\bm{r}}_{i,j}\cdot\hat{\bm{x}}}\\
    \bm{\chi}_{i,x} \hspace{-3mm}&=&\hspace{-3mm} -\mathrm{r}_0\frac{\ln\bm{\xi}_i}{2\mathrm{k}}\label{eqn:resultcalc}
\end{eqnarray}

\noindent respectively, where $\mathrm{k}_1$ is another constant controlling the speed of particle transition, just like $\mathrm{k}_0$. \textit{Further discussion is all based on this modified algorithm.}

\subsection{Detailed code and hardware requirements for each robot to achieve localization without direct communication}\label{sec:VPEcode}

Previous discussion focuses on the mathematical aspect of VPE localization algorithm, in which the swarm and the localization process is abstracted to a vector $\bm{\xi}$ and mathematical operations on it. However, how these operations are carried out by robots is an equally important problem, which determines whether the algorithm is distributed, whether the algorithm requires direct communication between robots, and sometimes the asymptotic time and memory complexity of the algorithm. In this section, hardware requirements of robots and pseudo-code executed by each robot are presented. By carrying out such code on every robot in the swarm synchronously, the whole swarm effectively executes the VPE localization algorithm in a distributed and direct communication-free way.

In order to achieve localization without direct communication between robots, we require that each robot possess following two abilities:

\begin{enumerate}
	\item [(i)] Sense ambient light intensity.
	\item [(ii)] Emit light with given intensity and angular distribution.
\end{enumerate}

\noindent These two requirements can easily be satisfied in systems based on unmanned cars, as it only requires an extra emitter-receiver ring mounted on each robot. A prototype of the ring is shown in Figure \ref{fig:indivrobot}.

In the localization process, instead of identifying other robots or sending digital signal to other robots, robots change the light intensity distribution in their vicinity and execute localization based on the variation of ambient light intensity they sensed. Thus this algorithm has no need for direct communication between robots. To some extent, we can even consider VPE localization algorithm as a sensing-based algorithm.

The key is how can Algorithm \ref{alg:1}, especially line \ref{alg:step3}, be realized using just the two abilities mentioned above. It is not hard to imagine that the process of sending VPs to other robots with amount defined by $\bm{\xi}_i\,P_{i,j}$ can be achieved by emitting light with angular distribution

\vspace{-2mm}
\begin{equation}
    I_1 (\hat{\bm{r}})=\bm{\xi}_i\;\mathrm{k}_1 e^{-\mathrm{k} \hat{\bm{r}}\cdot\hat{\bm{x}}}
\end{equation}
\vspace{-5mm}

\noindent and then letting robots sense the change in ambient light (for clarity, we name this process as `original process'). The change in ambient light $s_i$ sensed by robot $i$ is just the sum of contribution of all robots, which can be expressed by 

\vspace{-2mm}
\begin{equation}
s_i=\sum\limits_{j\in \bm{V}} \bm{\Gamma}_{j,i}\,\bm{\xi}_j\,\mathrm{k}_1 \,e^{-\mathrm{k} \hat{\bm{r}}_{j,i}\cdot\hat{\bm{x}}}=\sum\limits_{j\in \bm{V}} \bm{\xi}_j\,P_{j,i}
\label{eqn:recv}
\end{equation}

\noindent which is exactly what we wanted.

However, problem arises because robots cannot directly communicate with other robots: robots cannot know how many VPs they send to other robots in an iteration. Fortunately, this can be solved by introducing an additional process where robots emit light with angular distribution

\vspace{-2mm}
\begin{equation}
    I_2(\hat{\bm{r}})=\mathrm{k}_2 e^{\mathrm{k} \hat{\bm{r}}\cdot\hat{\bm{x}}}
\end{equation}

\noindent and then let robots sense the change in ambient light. In the equation, $\mathrm{k}_2$ is a constant controlling the intensity of light emitted. Similarly, the change in ambient light $c_i$ sensed by robot $i$ can be expressed as follows

\vspace{-2mm}
\begin{equation}
    c_i=\sum\limits_{j\in \bm{V}}\bm{\Gamma}_{j,i}\,\mathrm{k}_2\,e^{\mathrm{k} \hat{\bm{r}}_{j,i}\cdot\hat{\bm{x}}}
\end{equation}


\noindent Regarding $c_i$ we have the following relationship

\vspace{-2mm}
\begin{equation}
    \frac{\mathrm{k}_1 \bm{\xi}_i}{\mathrm{k}_2}\,c_i = \mathrm{k}_1 \bm{\xi}_i\,\sum\limits_{j\in \bm{V}} \bm{\Gamma}_{j,i} e^{\mathrm{k} \hat{\bm{r}}_{j,i}\cdot\hat{\bm{x}}} = \sum\limits_{j\in \bm{V}} \bm{\xi}_i\,P_{i,j}
    \label{eqn:send}
\end{equation}

\noindent Surprisingly, the right hand side represents the amount of VPs robot $i$ sends to other robots. Adding $\bm{\xi}_i$ on both sides of Equation \ref{eqn:recv} and subtracting Equation \ref{eqn:send} produces

\vspace{-4mm}
\begin{equation}
\begin{split}
    (1-\frac{c_i\,\mathrm{k}_1}{\mathrm{k}_2})\bm{\xi}_i+s_i=\bm{\xi}_i+\sum_{j\in \bm{V}}\bm{\xi}_j\,P_{j,i}-\bm{\xi}_i\,P_{i,j}\\=\bm{\xi}_i+\sum_{j\neq i}\bm{\xi}_j\,P_{j,i}-\bm{\xi}_i\,P_{i,j}
\end{split}
\label{eqn:VPsense}
\end{equation}

The right hand side of this equation is Equation \ref{eqn:VPevo} in Algorithm \ref{alg:1}, while the left hand side is only dependent to light intensity sensed by robot $i$ and amount of VPs owned by robot $i$. Equation \ref{eqn:VPsense} shows that robots can carry out localization using VPE localization algorithm without direct communication with any other robots if they have the two abilities listed above. For better comprehension, an illustration of the relationship between the additional process and the original process is given in Figure \ref{fig:illusbroadcast}, and a detailed pseudo-code executed by each robot is described in Algorithm \ref{alg:2}.

\begin{figure}[hbtp]
	\begin{center}
		\includegraphics[width=0.45\textwidth]{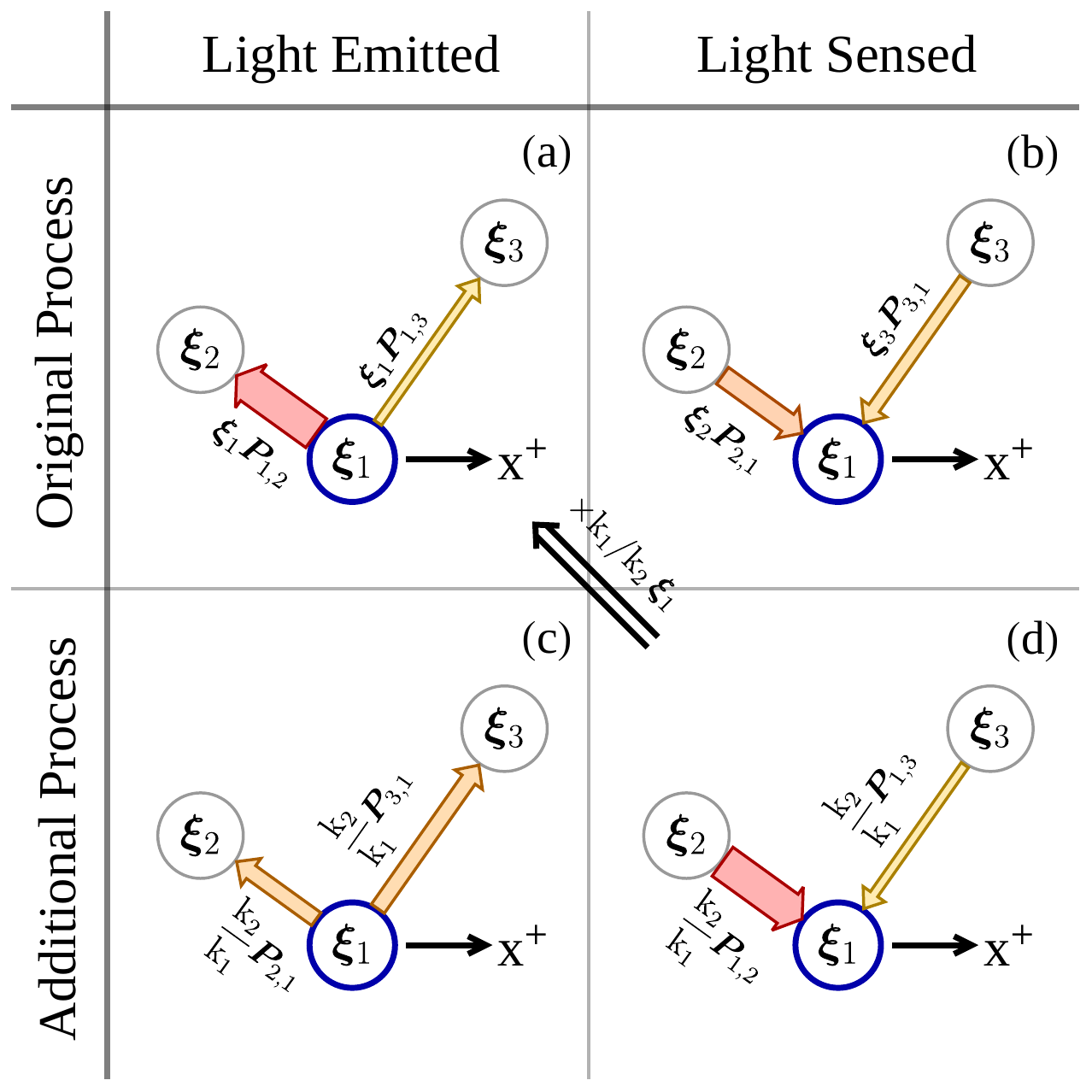}
		\caption{Illustration of robots' communication pattern. In the figure, each robot is represented by a circle and the robot we would like to focus on is colored in blue, the value inside each circle represents VP amount each robot owns and the value next to each arrow represents the contribution of the robot on the tail of the arrow to the light intensity measured by the robot on the head of the arrow. The figure shows that although a robot can only measure the light intensity it received, i.e. the second column, through the introduction of the additional process, a robot can obtain both the information of how much VPs it sends and how much VPs it receives in an iteration due to the similarity of light emitted in the original process (subfigure a) and light sensed in the additional process (subfigure d).}
		\vspace{-2mm}
		\label{fig:illusbroadcast}
	\end{center}
\end{figure}


\begin{breakablealgorithm}\label{alg:2}
	\caption{Pseudo-code for localization in $x$ axis}
	\begin{algorithmic}[1]
		\Require All robots:
		\begin{enumerate}
			\item [a.] Share a common $x$ direction.
			\item [b.] Can emit light with anisotropic angular distribution and sense ambient light intensity.
		\end{enumerate}
		\Ensure Localization result $\chi_x$ for all robots.
		\Notation
		\Notationx{$\xi_{+}, \xi_{-}$ - floating point numbers representing the amount of VPs possessed by a robot.}
		\Notationx{$n_{m}$ - iterations to calculate (predefined).}
		\Notationx{$\hat{\bm{r}}$ - light emitting direction.}
		\Notationx{$\mathrm{k}_1, \mathrm{k}_2$ - constants controlling the intensity of light emitted.}
		\Notationx{$\mathrm{k}$ - constant controlling the anisotropic distribution of light emitted.}
		\Notationx{$\mathrm{r}_0$ - weighed average distance between connected robots.}
		\ForEach{robot $i\in \bm{V}$.}
		\State $\xi_{+}\gets1$\label{alg:a2st}
		
		\MyState{Begin emitting light with angular distribution $I_2(\hat{\bm{r}})=\mathrm{k}_2 e^{\mathrm{k} \hat{\bm{r}}\cdot\hat{\bm{x}}}$}\label{alg:sst}
		
		\State $c\gets$ sensed ambient light intensity.
		\State Stop emission.\label{alg:sed}
		
		\For{$n=0\to n_m-1$}
			\State Time synchronization.
			\State Possible addition: \textbf{repeat} line \ref{alg:sst}-\ref{alg:sed}.
			\MyState{Begin emitting light with angular distribution $I_1 (\hat{\bm{r}})=\xi_{+}\;\mathrm{k}_1 e^{-\mathrm{k} \hat{\bm{r}}\cdot\hat{\bm{x}}}$.}			
			\State $s\gets$ sensed ambient light intensity.
		    \State Stop emission.
			\State $\xi_{+}\gets\left( 1-\frac{c\mathrm{k}_1}{\mathrm{k}_2} \right) \xi_{+} + s$
			\MyState{Possible addition: calibrate total VP amount using code given in Algorithm \ref{alg:calib} in appendix.}\label{alg:sst1}
		\EndFor\label{alg:a2ed}
		
		\MyState{\textbf{repeat} line \ref{alg:a2st}-\ref{alg:a2ed} with inverted $x$ axis and store the result as $\xi_{-}$.}
		\State $\chi_x\gets\mathrm{r}_0 \frac{\ln \xi_{-}-\ln\xi_{+}}{4\mathrm{k}}$.
	\end{algorithmic}
\end{breakablealgorithm}

\section{Analysis of VPE localization algorithm}\label{sec:anaVPEmethod}

VPE localization algorithm is analyzed in two aspects in this section: convergence and asymptotic time complexity. Theoretical derivation suggests that VPE localization algorithm will always converge to the same result regardless of initial state. Furthermore, the asymptotic time complexity of VPE localization algorithm applying on rectangular lattice is $\Theta(l)$ where $l$ represents the 1D size of the lattice.

\subsection{Convergence of VPE localization algorithm}

In following discussion, we consider a VPE localization algorithm applied on a connected homogeneous robot swarm consisting of $l$ robots under following two assumptions:

\begin{assum}\label{assum:1}
	\rm Robots can be considered as stationary in a localization process.
\end{assum}
\begin{assum}\label{assum:2}
	\rm All robots send and receive VP synchronously in VPE process.
\end{assum}

\begin{remark}
	\rm
	Assumption \ref{assum:1} is satisfied if robots execute Algorithm \ref{alg:2} much faster than they move or if robots take steps to update their positions and keep stationary when localization is in progress.
\end{remark}

The proof is based on five important properties of VPE localization algorithm:

\begin{enumerate}
	\item [(i)] In the initial state, VP amount on all robots are positive.
	\item [(ii)] Amount of VPs transmitted from robot $i$ to robot $j$ in an iteration is never negative.
	\item [(iii)] Amount of VPs transmitted from robot $i$ to robot $j$ in an iteration is linearly correlated to $\bm{\xi}_i$.
	\item [(iv)] A robot will not give all its VPs to neighbors in an iteration.
	\item[(v)] The swarm is connected.
\end{enumerate}

One can verify that properties (i,ii,iii,v) are automatically satisfied by Algorithm \ref{alg:2}, and property (iv) will be satisfied if for all robot $i$,

\vspace{-3mm}
\begin{equation}
	\sum_{j\neq i}P_{i,j}<1
	\label{eqn:crit}
\end{equation}

\begin{theorem}
VPE localization algorithm will always converge if Assumption \ref{assum:1} and \ref{assum:2} and Equation \ref{eqn:crit} hold. Furthermore, the localization result will converge to a point which is relevant to the configuration of robot swarm, only shift as a whole when the total amount of VPs changes, and is irrelevant to the distribution of VPs in the initial state.
\end{theorem}

\begin{proof}
We should first convert theorem into a mathematical form. Assumptions \ref{assum:1}, \ref{assum:2}, and property (iii) mean VPE process could be considered as a series of linear transformations on VP vector described by

\vspace{-2mm}
\begin{equation}
	\bm{\xi}^{(n+1)}=\bm{T}\,\bm{\xi}^{(n)}
	\label{eqn:iter}
\end{equation}
\vspace{-5mm}

\noindent where

\vspace{-4mm}
\begin{equation}
\bm{T}\hspace{-.5mm}=\hspace{-.5mm}
\left(\hspace{-1.5mm}
\begin{array}{ccccc}
1\hspace{-1mm}-\hspace{-1mm}\sum\limits_{i=2}^{l}P_{1,i} & P_{2,1}& \cdots & P_{l,1}\\
P_{1,2} & 1\hspace{-1mm}-\hspace{-1mm}\sum\limits_{i=1,i\neq 2}^{l}P_{2,i}& \cdots & P_{l,2}\\
\vdots & \vdots & \ddots & \vdots\\
P_{1,l} & P_{2,l} & \cdots & 1\hspace{-1mm}-\hspace{-1mm}\sum\limits_{i=1}^{l-1}P_{l,i}
\end{array}
\hspace{-1.5mm}\right)
\label{eqn:Tform}
\end{equation}

\noindent is the VP transition matrix. Similarly, properties (ii) and (iv) can be interpreted as follows: for all $i\in \bm{V}$ and $j\in \bm{V}, j\neq i$,

\vspace{-2mm}
\begin{equation}
	\left\{
	\begin{aligned}[c]
	\bm{T}_{i,j}\geq0\\
	\bm{T}_{i,i}>0
	\end{aligned}
	\right.
	\label{eqn:tprop}
\end{equation}
\vspace{-2mm}

\noindent and property (v) means there is no permutation matrix $\bm{U}$ which satisfies

\vspace{-2mm}
\begin{equation}
\bm{U}\bm{T}\bm{U}^\mathrm{T}\hspace{-1mm}=\hspace{-1mm}
\left(\hspace{-2mm}
\begin{array}{cc}
\bm{Q} & \hspace{-2.5mm}\bm{R}\\
\bm{0} & \hspace{-2.5mm}\bm{S}
\end{array}\hspace{-2mm}
\right)
\label{eqn:noperm}
\end{equation}
\vspace{-2mm}

\noindent where $\bm{Q}$ and $\bm{S}$ are square matrices with non-zero dimension. Equivalently, this means matrix $\bm{T}$ is irreducible.

Because $\bm{T}$ is non-negative, irreducible and for all $i$, $\bm{T}_{i,i}>0$ which means the system is aperiodic, matrix $\bm{T}$ is primitive (i.e. there exists a natural number $m$ which lets $\bm{T}^m$ be all positive), thus Perron-Frobenius theorem could be applied, which indicates the following three properties:

\begin{enumerate}
	\item [(i)] the Perron–Frobenius eigenvalue (eigenvalue of matrix $\bm{T}$ with largest absolute value) $r$ satisfies

	\vspace{-2mm}
    \begin{equation}
        1=\min\limits_i\sum\limits_{j=1}^l| \bm{T}_{j,i}|\leq r\leq \max\limits_i\sum\limits_{j=1}^l| \bm{T}_{j,i}|=1
	\end{equation}
	\vspace{-3mm}
    
    \item [(ii)] The left and right eigenspace associated to $r$ are one-dimensional. In further discussion, $\bm{v}$ will denote the right eigenvector which satisfies $\sum_{i=1}^l\bm{v}_i=1$ and $\bm{w}^\mathrm{T}$ will denote $\bm{1}_l^\mathrm{T}$ which is evidently a left eigenvector of $\bm{T}$ and satisfies $\bm{w}^\mathrm{T}\bm{v}=1$.

    \item [(iii)] the limit of $\bm{T}^n$ satisfies
	
	\vspace{-2mm}
    \begin{equation}
        \lim\limits_{n\to \infty} \frac{\bm{T}^n}{r^n}=\bm{v}\bm{w}^\mathrm{T}
	\end{equation}
	\vspace{-5mm}
\end{enumerate}

\noindent Using all three properties, we can deduce that

	\vspace{-2mm}
    \begin{equation}
        \bm{\xi}^{(\infty)}=\lim\limits_{n\to \infty} \bm{T}^n \bm{\xi}^{(0)} =\bm{v}(\bm{w}^\mathrm{T} \bm{\xi}^{(0)})=\bm{v}\sum_{j=1}^l\bm{\xi}_j^{(0)}
        \label{eqn:VPinf}
    \end{equation}
	\vspace{-2mm}


\noindent Plugging Equation \ref{eqn:VPinf} into Equation \ref{eqn:resultcalc}, the converging point of localization result is obtained as follows

\vspace{-2mm}
\begin{equation}
    \bm{\chi}_{i,x}=-\frac{\mathrm{r}_0}{2\mathrm{k}} (\ln\bm{v}+\ln(\sum_{j=1}^l\bm{\xi}_j^{(0)}))
\end{equation}
\vspace{-2mm}

\noindent which means the localization result will converge to a point which is relevant to the configuration of robot swarm, only shift as a whole when the total amount of VPs changes, and is irrelevant to the distribution of VPs in the initial state.
\end{proof}

It is worth noting that if Algorithm \ref{alg:1} is used, i.e. $\bm{P}_{i,j} =\mathrm{k_0} \bm{\Gamma}_{i,j} e^{-\mathrm{k}\vec{\bm{r}}_{i,j}\hat{\bm{x}}}$, then $\bm{v}=\dfrac{e^{-2\mathrm{k}x_i}}{\sum_{j=1}^l e^{-2\mathrm{k}x_j}}$, thus in the equilibrium state the localization result will satisfy $\bm{\chi}_{i,x}=x_i+(\ln(\sum_{j=1}^l\bm{\xi}_j^{(0)})-\ln(\sum_{j=1}^l e^{-2\mathrm{k}x_j}))$, which means the localization result will be accurate after a translation to align the centroid of localization result and the centroid of the swarm.


\begin{remark}\label{rem:convergence}
	\rm
	By Observing the convergence behavior of $\bm{\xi}^{(n)}-\bm{\xi}^{(0)}$, one could notice that the difference converges with a rate dominated by $|\lambda_2|^n$ where $\lambda_2$ is the eigenvalue of $\bm{T}$ which has the second largest absolute value. This means that error converges linearly, and the rate of convergence is determined by $|\lambda_2|$.
\end{remark}

\begin{remark}\label{rem:calib}
	\rm
	Due to the exponential shrinking behavior of components corresponding to eigenvalues except for the Perron–Frobenius eigenvalue, the algorithm will be numerically stable and error resistant if a normalization is performed on $\bm{\xi}$ regularly to prevent changes of total VP amount (i.e. the component corresponding to Perron–Frobenius eigenvalue).
	
	A typical normalization method could be

	\vspace{-3mm}
	\begin{equation}
		\mathcal{N}(\bm{\xi}^{(n)})=\frac{l\,\bm{\xi}^{(n)}}{\sum_{j\in \bm{V}} \bm{\xi}_j^{(n)}}
	\end{equation}

	\noindent which keeps the VP amount per robot fixed at 1. Detailed pseudo-code for normalization is provided in Algorithm \ref{alg:calib} in appendix.
	
	Through regular normalization, the convergence of the entire algorithm under noise could be ensured. Experimental demonstration of the effectiveness of such normalization method against noise and drift is presented in Figure \ref{fig:calib}.
	
\end{remark}

\subsection{Asymptotic time complexity of VPE localization algorithm}\label{sec:timecomplexity}

In this section, a semi-quantitative analysis about the asymptotic time complexity of VPE localization algorithm is provided. Because robots update their position estimation iteratively in VPE localization algorithm and the time required to execute an iteration does not vary with the size of the swarm, in following discussion we will use iterations required to reach a certain accuracy as the measurement of execution time.

In following discussion, we mainly concern a scenario as illustrated in Figure \ref{fig:1ddetail}, where $l$ robots are distributed evenly on a 1D straight line with a gap of 1 unit and can only exchange VPs with immediate neighbors. Furthermore, the localization accuracy $\delta^{(n)}$ in a certain iteration $n$ is defined as the maximum discrepancy between the localization result in this iteration and the converging point of localization result, i.e.

\vspace{-2mm}
\begin{equation}
    \delta^{(n)}=\max\limits_i |\bm{\chi}_{i,x}^{(n)}-\bm{\chi}_{i,x}^{(\infty)}|
\end{equation}
\vspace{-3mm}

\noindent where $\bm{\chi}_{i,x}^{(n)}$ represents the localization result calculated with the VP distribution $\bm{\xi}^{(n)}$ in iteration $n$ using Equation \ref{eqn:resultcalc}, that is 

\begin{equation}
    \bm{\chi}_{i,x}^{(n)}=-\mathrm{r}_0\frac{\ln\bm{\xi}^{(n)}_i}{2\mathrm{k}}
\end{equation}

\begin{figure}[htbp]
	\begin{center}
		\includegraphics[width=.45\textwidth]{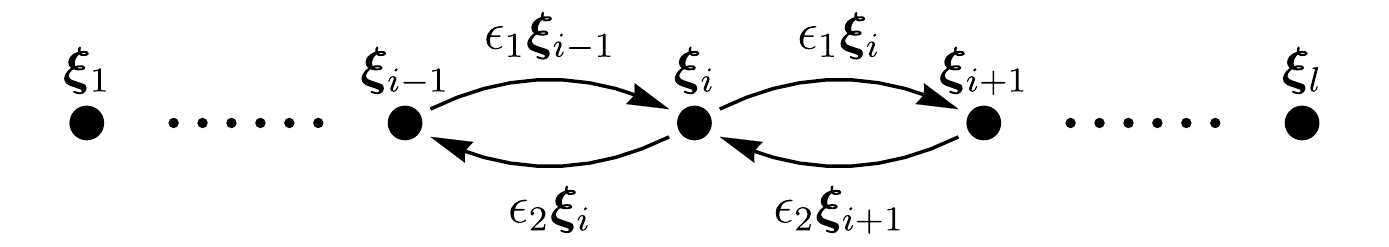}
		\caption{A simplified model where $l$ robots represented by black dots can exchange VPs with their immediate neighbors according to previously mentioned rules. $\mathrm{\epsilon}_1$ and $\mathrm{\epsilon}_2$ are two transition parameters.}
		\label{fig:1ddetail}
	\end{center}
\end{figure}

\begin{theorem}
In a 1D robot swarm described above, the iteration $n$ required to achieve an accuracy of $\delta^{(n)}\leq\delta_0\leq\frac{1}{2}$ is $\Theta(l)$ where $l$ is the number of robots.
\end{theorem}

\begin{proof}
In this case, VP transition matrix $\bm{T}$ is




\vspace{-2mm}
\begin{equation}
	\bm{T}_{l\times l}\hspace{-1mm}=\hspace{-1mm}
	\left(\hspace{-2mm}
	\begin{array}{cccccc}
	1-\epsilon_1 &\hspace{-2mm} \epsilon_2 \\
	\epsilon_1 &\hspace{-2mm} 1-\epsilon_1-\epsilon_2 &\hspace{-1mm} \epsilon_2\\
	&\hspace{-2mm}\ddots&\hspace{-1mm}\ddots&\hspace{-1mm}\ddots\\
	&&\hspace{-1mm}\epsilon_1 & \hspace{-1mm}1-\epsilon_1-\epsilon_2 & \hspace{-2mm}\epsilon_2\\
	&&&\hspace{-1mm}\epsilon_1 & \hspace{-2mm}1-\epsilon_2
	\end{array}\hspace{-2mm}
	\right)
\end{equation}

\noindent where $\mathrm{\epsilon}_1=\mathrm{k}_1 e^{-\mathrm{k}}$, and $\mathrm{\epsilon}_2=\mathrm{k}_1 e^{\mathrm{k}}$. Without loss of generality, further assume $\mathrm{\epsilon}_1<\mathrm{\epsilon}_2$ and name $\gamma=\frac{\mathrm{\epsilon}_2}{\mathrm{\epsilon}_1}$. Using these notations, linearly independent eigenvectors $\{\bm{v}_q\}\ (q=1,2,\cdots,l)$ and corresponding eigenvalues $\{\lambda_q\}\ (q=1,2,\cdots,l)$ of matrix $\bm{T}$ can be expressed as follows

\begin{equation}
\left\{
\begin{aligned}[c]
	\lambda_1&=1\\
	 (\bm{v}_1)_i&=\gamma^{-i}\\
	 \lambda_p&=1-\mathrm{\epsilon}_1-\mathrm{\epsilon}_2+2\sqrt{\mathrm{\epsilon}_1 \mathrm{\epsilon}_2} \cos{\frac{(p-1)\pi}{l}}\\
	 (\bm{v}_p)_i&=\gamma^{-\frac{i}{2}} \cos{\left(\frac{i(p-1)\pi}{l}+\mathrm{\phi}_p\right)}
\end{aligned}
\right.
\label{eqn:eig}
\end{equation}

\noindent where index $p=2,3,\cdots,l$, and $\mathrm{\phi}_p$ satisfy

\vspace{-2mm}
\begin{equation}
	\cos(\mathrm{\phi}_p)=\sqrt{\gamma}\cos{\left(\frac{(p-1)\pi}{l}+\mathrm{\phi}_p\right)}
\end{equation}

For clarity, matrices $\bm{S}$ and $\bm{S}'$ are introduced with definition as follows.

\vspace{-4mm}
\begin{eqnarray}
    \bm{S}'&=&\mathop{\mathrm{diag}}{\left\{\gamma^{\frac{1}{2}},\gamma^{1},\cdots,\gamma^{\frac{l}{2}}\right\}}\\
    \bm{S}&=&\mathop{\mathrm{diag}}{\left\{\gamma^{1},\gamma^{2},\cdots,\gamma^{l}\right\}}=\bm{S}'^2
\end{eqnarray}
\vspace{-4mm}

\noindent Because $\bm{S}'\bm{T}\bm{S}'^{-1}$ is a symmetric matrix, its eigenvectors are orthogonal, thus we have the following relationship for eigenvectors $\{\bm{u}_q\}\ (q=1,2,\cdots,l)$ of matrix $\bm{S}\bm{T}\bm{S}^{-1}$

\vspace{-4mm}
\begin{eqnarray}
    &\bm{u}_1 = \bm{1}_l\\
    &\bm{u}_i^{\mathrm{T}}\,\bm{S}^{-1}\,\bm{u}_j=0\ \ (i\neq j)
    \label{eqn:ortho}
\end{eqnarray}

\noindent Further define operation $\|\bm{v}\|_S$ which maps a vector of dimension $l$ to a non-negative real number by

\vspace{-2mm}
\begin{equation}
    \|\bm{v}\|_S=\sqrt{\bm{v}^{\mathrm{T}}\,\bm{S}^{-1}\,\bm{v}}
\end{equation}
\vspace{-4mm}

\noindent It is evident that for arbitrary vector $\bm{v}$ of dimension $l$, $\max_i |\bm{v}_i|\leq\gamma^{\frac{l}{2}}\|\bm{v}\|_S$.


Now we proceed to discuss the converging behavior of $\delta^{(n)}$ which can be expressed using $\bm{\xi}$ as follows

\begin{equation}
    \delta^{(n)}=\max\limits_i \frac{\left|\ln\left(1+\frac{\gamma^i\bm{\xi}^{(n)}_i-\gamma^i\bm{\xi}^{(\infty)}_i}{\gamma^i\bm{\xi}^{(\infty)}_i}\right)\right|}{\ln\gamma}
\end{equation}


\noindent Noticing that $\gamma^i\bm{\xi}^{(\infty)}_i$ is the same for all index $i$, a sufficient condition for $\delta^{(n)}<\delta_0$ can be expressed as

\vspace{-2mm}
\begin{equation}
    \max\limits_i \left|\gamma^i\,\bm{\xi}^{(n)}_i-\gamma^i\,\bm{\xi}^{(\infty)}_i\right| \leq \psi(1-\gamma^{-\delta_0})
    \label{eqn:maxmod}
\end{equation}

\noindent where $\psi=\gamma^i\bm{\xi}^{(\infty)}_i$. At the same time, $\gamma^i\bm{\xi}^{(n)}_i$ is also the $i$th component of $\bm{S}\,\bm{\xi}^{(n)}$, thus by using the property of $\|\cdot\|_S$, a sufficient condition for inequation \ref{eqn:maxmod} is

\vspace{-2mm}
\begin{equation}
    \|\bm{S}\,\bm{\xi}^{(n)}-\bm{S}\,\bm{\xi}^{(\infty)}\|_S \leq \gamma^{-\frac{l}{2}}\psi(1-\gamma^{-\delta_0}))
    \label{eqn:suffcond}
\end{equation}
\vspace{-4mm}

\noindent If $\bm{S}\,\bm{\xi}^{(n)}$ is decomposed using eigenvectors of $\bm{S}\bm{T}\bm{S}^{-1}$ as

\vspace{-2mm}
\begin{equation}
    \bm{S}\,\bm{\xi}^{(n)}=(\bm{S}\bm{T}\bm{S}^{-1})^n(\bm{S}\,\bm{\xi}^{(0)})=\sum\limits_{i=1}^{l}a^{(n)}_i \bm{u}_i
\end{equation}
\vspace{-2mm}

\noindent where $a^{(n)}_i$ is the component of $\bm{S}\,\bm{\xi}^{(n)}$ on $\bm{u}_i$. The following relationship of coefficient $a^{(n)}_i$ could be found

\vspace{-2mm}
\begin{equation}
    a^{(n)}_i=\lambda_i^n a^{(0)}_i
\end{equation}
\vspace{-5mm}

\noindent which means

\vspace{-2mm}
\begin{equation}
    \bm{S}\,\bm{\xi}^{(n)}-\bm{S}\,\bm{\xi}^{(\infty)}=\sum\limits_{i=2}^{l}\lambda_i^n a^{(0)}_i \bm{u}_i
\end{equation}
\vspace{-2mm}

\noindent remembering Equation \ref{eqn:ortho}, left hand side of inequation \ref{eqn:suffcond} can be written as

\vspace{-4mm}
\begin{equation}
\begin{split}
     \|\bm{S}\,\bm{\xi}^{(n)}-\bm{S}\,\bm{\xi}^{(\infty)}\|_S=\sqrt{\sum\limits_{i=2}^{l}\lambda_i^{2n} (a^{(0)}_i)^2 \|\bm{u}_i\|_S^2}\\
     \leq \lambda_2^n \|\bm{S}\,\bm{\xi}^{(0)}-\bm{S}\,\bm{\xi}^{(\infty)}\|_S
\end{split}
\label{eqn:estimate}
\end{equation}

\noindent Combining Equation \ref{eqn:suffcond} and \ref{eqn:estimate} and noticing that $\lambda_2$ is the second largest eigenvalue of $\bm{T}$, the upper limit of iteration required can be given as

\vspace{-2mm}
\begin{equation}
    -\left.\ln \frac{\gamma^{\frac{l}{2}}\|\bm{S}\,\bm{\xi}^{(0)}-\bm{S}\,\bm{\xi}^{(\infty)}\|_S}{\psi(1-\gamma^{-\delta_0})}\right/\ln \lambda_2
\end{equation}
\vspace{-2mm}

\noindent By evaluating the upper bound of numerator and substituting $\lambda_2$ by its upper bound of $1-\mathrm{\epsilon}_1-\mathrm{\epsilon}_2+2\sqrt{\mathrm{\epsilon}_1 \mathrm{\epsilon}_2}$, a simpler form of upper limit $n_{max}$ could be obtained

\vspace{-2mm}
\begin{equation}
    n_{max}=-\frac{l\ln\gamma-\ln(\gamma-1)-2\ln(1-\gamma^{-\delta_0})}{2\ln(1-\mathrm{\epsilon}_1-\mathrm{\epsilon}_2+2\sqrt{\mathrm{\epsilon}_1 \mathrm{\epsilon}_2})}
    \label{eqn:nmax}
\end{equation}

\noindent The denominator is a constant independent of $l$, while the numerator is dominated by $l\ln\gamma$ when $l$ is large, thus the asymptotic time complexity for achieving an accuracy of $\delta^{(n)}\leq\delta_0$ is at most $O(l)$.
	
Furthermore, in this system, $\bm{\xi}^{(n)}_i\ (n<i,n\leq l-i)$ takes the form of $f^{(n)}(\bm{\xi}^{(0)}_{i-n},\cdots,\bm{\xi}^{(0)}_{i+n})$ and $f^{(n)}$ is independent of $i$. Thus, when $n\leq\frac{l}{2}-1$, $\bm{\xi}^{(n)}_{\lfloor l/2\rfloor}$ and $\bm{\xi}^{(n)}_{\lfloor l/2\rfloor+1}$ will be identical because $\bm{\xi}^{(0)}_i=1$. This means that these two robots, though 1 unit apart, share the same localization result when $n\leq\frac{l}{2}-1$, thus the maximum localization error must be larger than $1/2$, which indicates that the asymptotic time complexity to achieve an accuracy higher than $\frac{1}{2}$ is at least $\Omega(l)$.

As the asymptotic time complexity for achieving an accuracy of $\delta^{(n)}\leq\delta_0$ is at most $O(l)$ and at least $\Omega(l)$, it is $\Theta(l)$.

\end{proof}

\begin{remark}
	\rm
	In fact, regardless of the method used, the asymptotic time complexity of localization in a swarm is physically limited by the transmission speed of information which is proportional to the length of the swarm, that means no algorithm with asymptotic complexity lower than $\Theta(l)$ is possible.
\end{remark}

\begin{remark}
    In higher dimensional systems, VPE localization algorithm offers stunning theoretical performance. For example, consider a localization task in 2D and 3D robot swarms as illustrated in Figure \ref{fig:grids}, in a rectangular 2D robot swarm with $l\times l$ robots or a 3D swarm with $l\times l\times l$ robots, the asymptotic time complexity of localization will still be $\Theta(l)$ (due to translational symmetry, asymptotic time complexity of localization in such a swarm is identical to that of 1D scenario) instead of $\Theta(l^2)$ or $\Theta(l^3)$ which is the lower limit of all centralized algorithm (time required for storing information transmitted from all robots in memory).
    
    \begin{figure}[htbp]
	\centering
	\setcounter{subfigure}{0}
	
	\subfigure[2D]{
		\begin{minipage}[t]{0.4\linewidth}
			\centering
			\includegraphics[width=\linewidth]{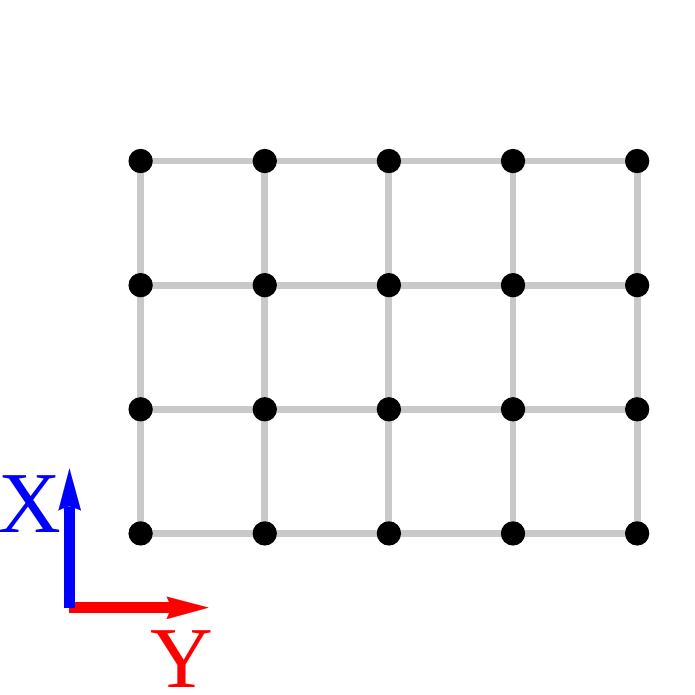}
			\label{fig:sub2d}
		\end{minipage}
	}
	\subfigure[3D]{
		\begin{minipage}[t]{0.4\linewidth}
			\centering
			\includegraphics[width=\linewidth]{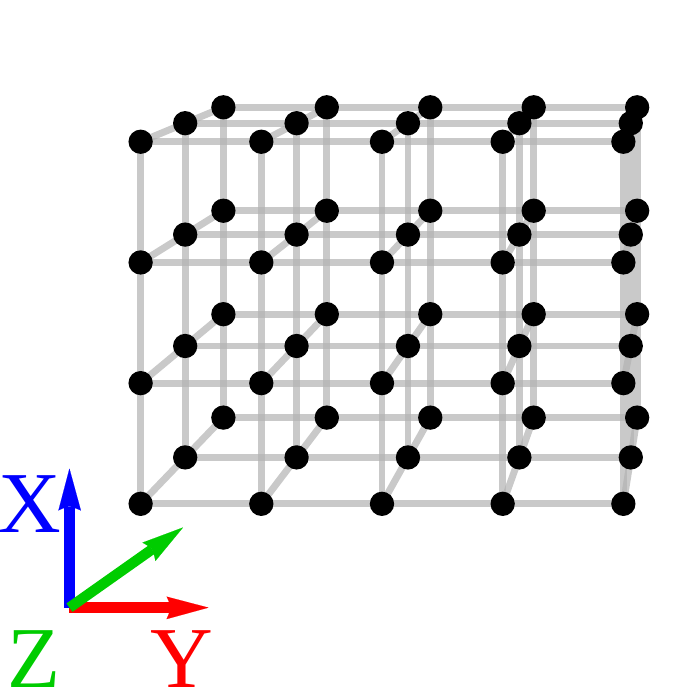}
			\label{fig:sub3d}
		\end{minipage}

	}
	\centering
	\caption{Robot swarm in higher dimension. Black dots represent robots and lines represent connections between robots.}
	\label{fig:grids}
	\vspace{-2mm}
\end{figure}
\end{remark}

\begin{remark}\label{rem:ndef}
	\rm
	Iterations required to achieve certain accuracy across the swarm will vary when the distribution of robots or the connectivity between robots changes, however some qualitative trends regarding these factors can be found:
	
	\begin{itemize}
		\item Iterations required decreases when there are stronger connections and more long-range connections between robots.
		\item Iterations required increases when the swarm is larger or more sparsely distributed.
	\end{itemize}
	
	
	
\end{remark}

\begin{remark}
	\rm
	Often we need to perform localization on moving robots, e.g. in shape formation tasks. In such cases, robots' position will not vary much between two consecutive localization tasks. This means that by setting the initial VP amount to the final VP amount in previous localization tasks, initial error $\bm{\xi}^{(0)}-\bm{\xi}^{(\infty)}$ could be greatly reduced, which can greatly accelerate the localization process.
	
	This method is used in simulations conducted in Section \ref{sec:sfresult}, the result of which shows up to 90\% time could be saved compared with a localization procedure starting from scratch.
\end{remark}


\section{Localization Verification}\label{sec:locresult}

In this section, localization results in simulations and experiments using VPE localization algorithm are presented and analyzed.

\subsection{Results in simulations}

Simulations are first carried out to provide localization results in $x$ axis only on 1D and 2D robot swarms in order to investigate the asymptotic time complexity of localization tasks executed with different robot distribution patterns.

Four robot distribution patterns are used in simulations:

\begin{enumerate}
	\item[(1)] 1D distribution: robots evenly distributed on a line with spacing of 1.
	\item[(2)] 2D distribution: robots are located in a square whose edge is parallel to $x$ or $y$ axis.
	\item[(3)] 2D distribution: robots are located in a square whose diagonal is parallel to $x$ or $y$ axis.
	\item[(4)] 2D distribution: robots are located in a annulus whose outer radius is twice the inner radius.
\end{enumerate}

\noindent If robots are distributed in 2D, their position will be selected randomly while keeping the distance between adjacent robots at approximately 1. The size of the robot swarm is controlled by a dimensionless variable named as `size factor' which reflects the span of the swarm on $x$ axis. In 1D scenarios size factor is defined as the number of robots, and in 2D scenarios size factor is defined as the square root of number of robots. Robot swarms with size factor between 2 and 100 are considered in the simulation. In simulations with 1D distribution of robots, conditions where maximum transmission distance of light (light will be too dim to be detected over this distance) is $1.5$, $2.5$ or $3.5$ are evaluated respectively. Parameters used in simulations with 1D distribution of robots are $\mathrm{k}_1=0.05,\,\mathrm{k}=0.15$ and $r_0=(\lfloor d\rfloor+1)/2$ where $d$ denote the transmission distance of light. In simulations with 2D distribution of robots, maximum transmission distance of light is set to 2.5 and $\mathrm{k}_1=0.05,\,\mathrm{k}=0.15,\,\mathrm{r}_0=1.72$ are used. In all simulations, localization results are considered as `converged' when the maximum discrepancy between current localization result and the localization result of equilibrium state is less than 0.1 unit.

Accuracy of localization result is provided in Figure \ref{fig:simerr}. In most simulations with size factor larger than $8$, localization result in the equilibrium state has average localization error less than $0.15$. The reason why localization has greater error with smaller swarm is because $\mathrm{r}_0$ chosen before is optimized for large swarms. If $\mathrm{r}_0$ is a variable and is optimized such that localization error is minimized, then in all scenarios localization error will be less than $0.12$ unit (see subfigure (b)).

Iterations required for VPE localization algorithm to converge in swarm with different size and distribution are shown in Figure \ref{fig:convt}. It could be observed that the iterations required to reach a certain accuracy are linearly correlated to the size factor of the shape instead of number of robots. Furthermore, By comparing group Line\underline{~}1 to Line\underline{~}3, one can observe that less iterations are required to reach a certain accuracy when robots can sense light emitted by further robots. It is also worth mentioning that in the case where 10000 robots are located in a square (size factor = 100), merely 6000 iterations are required for the algorithm to converge, the number can be even more staggering with larger swarm or swarm in 3D. This phenomenon indicates that this algorithm is especially suitable for large swarm.

\begin{figure}
    \centering
	\includegraphics[width=\linewidth]{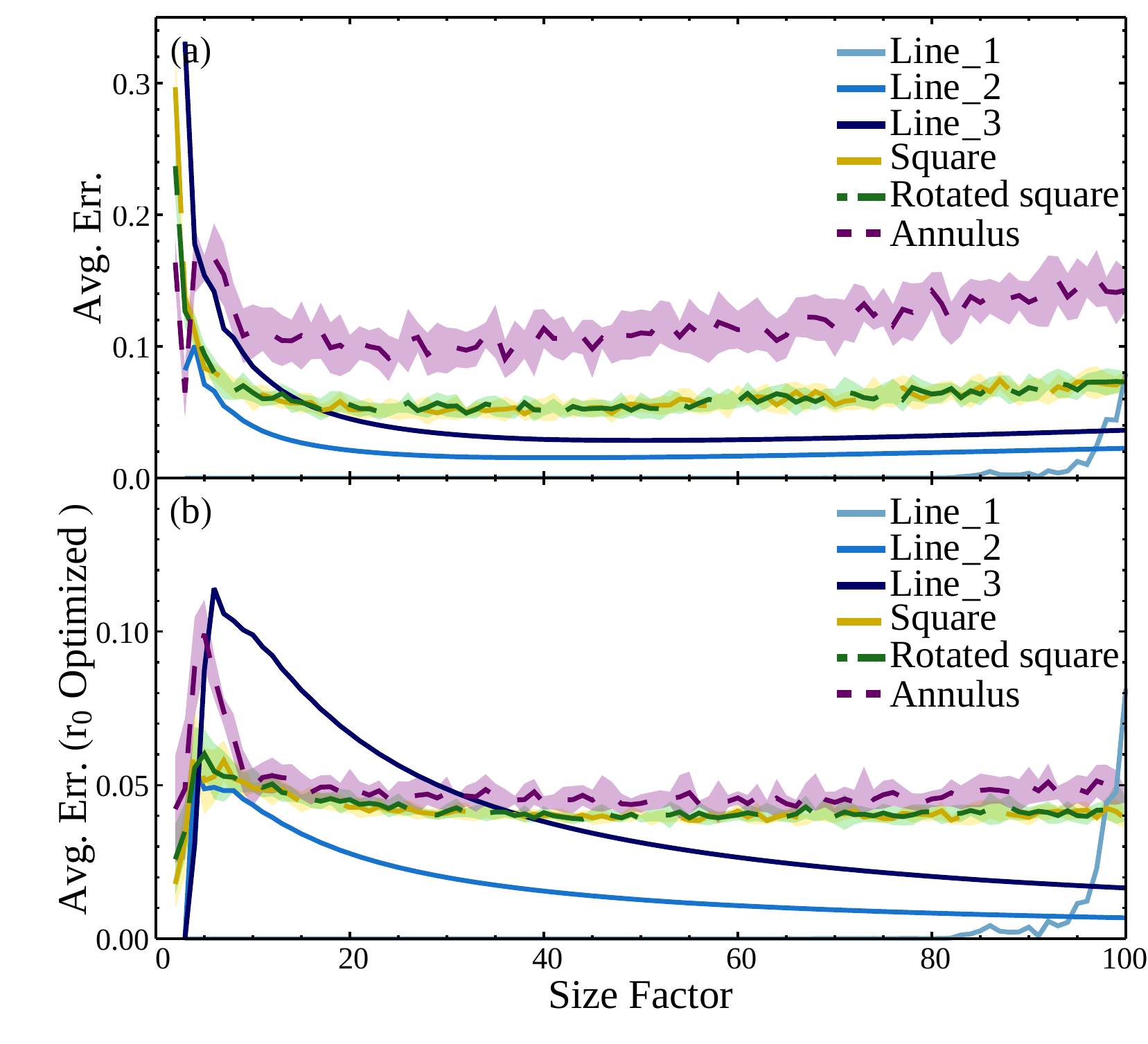}
	\vspace{-7mm}
	\caption{Average localization error in the equilibrium state. Subfigure (a) displays the localization error with previously described settings while subfigure (b) displays the localization error when $\mathrm{r}_0$ is a variable and is optimized such that localization error is minimized. Curve Line\underline{~}1 to Line\underline{~}3 correspond to distribution (1) where maximum transmission distance of light is 1.5, 2.5 and 3.5. Curve Square, Rotated square, and Annulus correspond to configurations (2), (3), (4) respectively and is the average of 10 runs. Shades behind these three curves indicate the standard deviation of the data. Note that the increase in error in configuration Line\underline{~}1 at large size factor is due to numerical error.}
	\vspace{-2mm}
    \label{fig:simerr}
\end{figure}

\begin{figure}
    \centering
	\includegraphics[width=\linewidth]{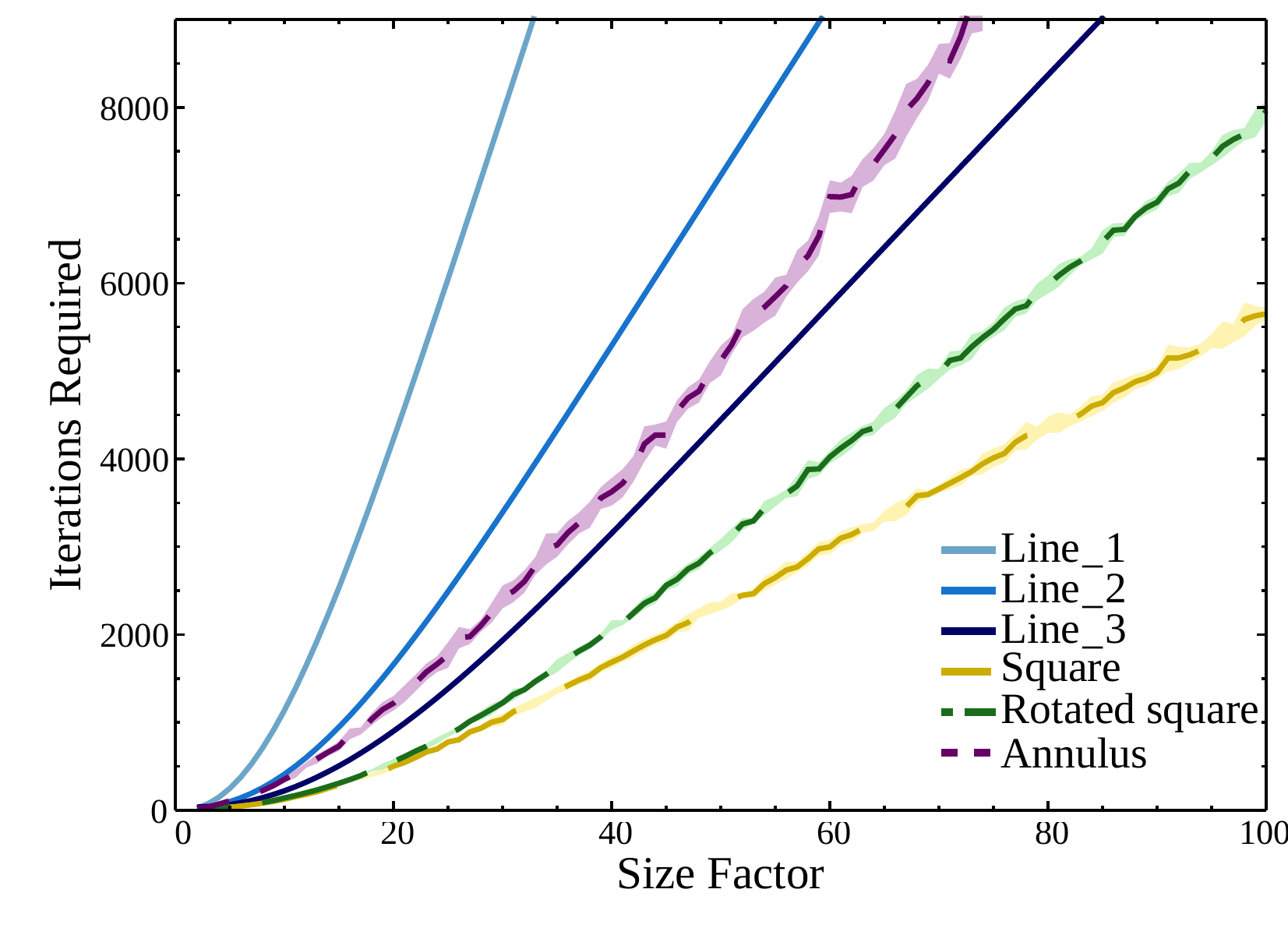}
	\vspace{-7mm}
	\caption{Iterations required to converge. The definition of curves are the same as in Figure \ref{fig:simerr}.}
	\vspace{-4mm}
    \label{fig:convt}
\end{figure}

In previous discussion, we focused on obtaining $x$ coordinates of robots only. By executing VPE localization algorithm twice in $x$ and $y$ axis respectively, each robot can obtain a 2D localization result. To provide better comprehension of the procedure of VPE localization algorithm and its capabilities, 2D localization is achieved in three scenarios:

\begin{enumerate}
	\item[(1)] Robots are located in a square whose edge is parallel to $x$ or $y$ axis.
	\item[(2)] Robots are located in a annulus whose outer radius is twice the inner radius.
	\item[(3)] Same as (2) except 10\% measurement noise is added to the sensor of robots.
\end{enumerate}

The simulation results are shown in Figure \ref{fig:simlocres}. All parameters are the same as those used in previous simulations with 2D distribution of robots.

\begin{figure*}[htbp]
	\centering
	\includegraphics[width=\linewidth]{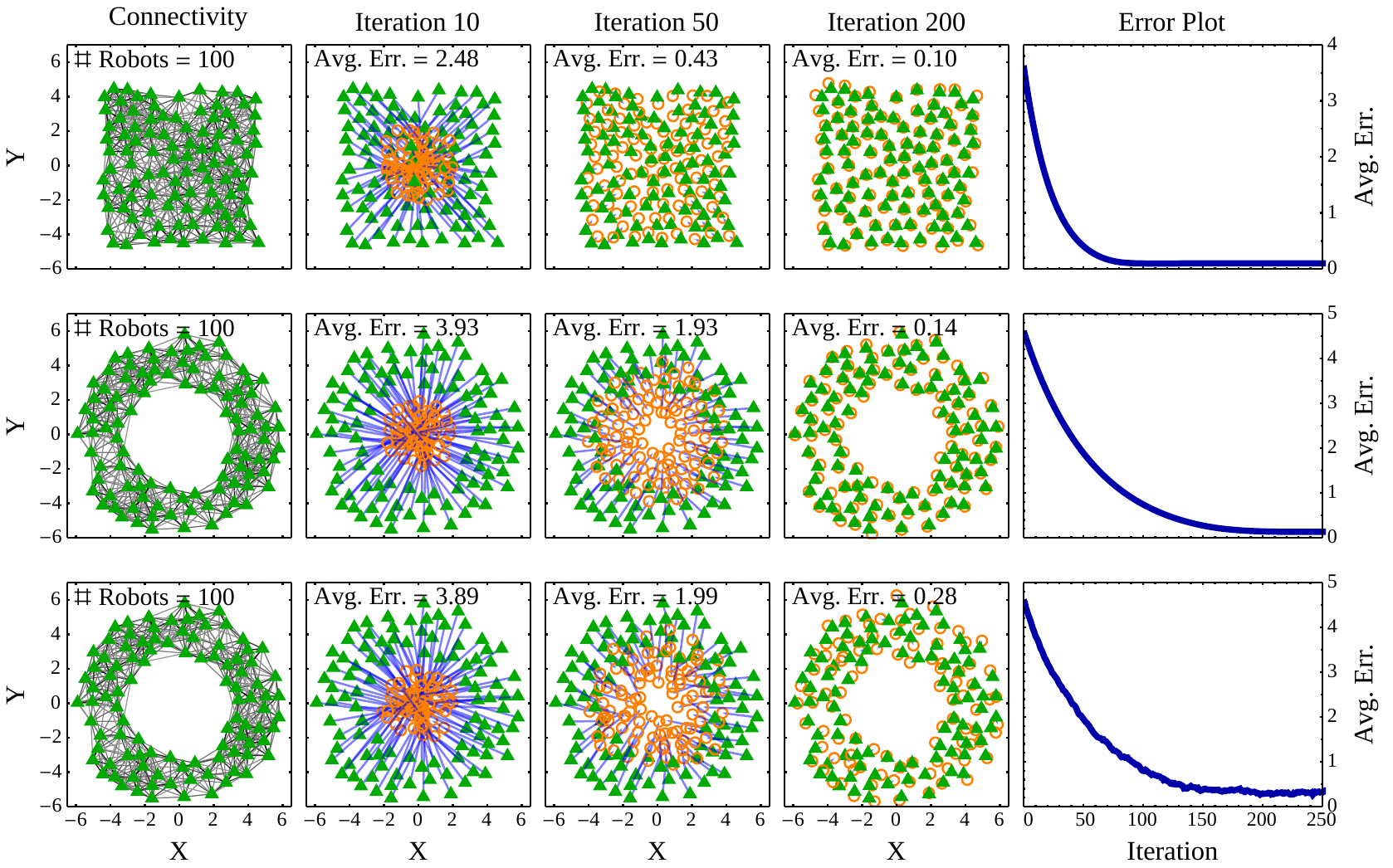}
	\centering
	\vspace{-5mm}
	\caption{Localization procedure with different robot distribution. The first column shows the connectivity of the whole swarm, where each green triangle marks the actual position of a robot and each line represents that light emitted by one robot can be sensed by another, the darker the line is, the stronger the light sensed. Columns 2,3,4 display the localization results in different iterations of the localization process, where green triangles still represent the actual position of robots, yellow circles represent localization results of the robot, and blue arrows represent position error of the robot. The last column shows the trend of localization error versus iteration. In all figures, both the robots' centroid and the localization result's centroid are set to $\{0,0\}$ through translation.}
	\vspace{-3mm}
	\label{fig:simlocres}
\end{figure*}

One can see that in Figure \ref{fig:simlocres}, localization results gradually converge in less than 200 iterations in all three scenarios and the average error of localization is only approximately 0.1 unit, which is sufficiently accurate for most tasks. Also, in the last case where 10\% of noise is added to all sensors readings, our algorithm still managed to achieve an average localization error of 0.5 units, which verified that VPE localization algorithm is resistive to noise. It is also worth noting that the localization results' centroid is within 1 unit of the actual centroid of robots in all three scenarios, which is ideal for tasks like shape formation.

\subsection{Results in experiments}

To further verify that the VPE localization algorithm is suitable for real applications with strong environmental noise and interference, experiments with 52 low-cost (approximately \$35 each), self-designed robots were conducted.

Figure \ref{fig:env} shows the test environment and all robots involved in the experiment, and Figure \ref{fig:indivrobot} shows the detailed view of a robot. Robots move on the ground and a camera records the movements of robots and the light signal feedback sent by them.

\begin{figure}[htbp]
	\begin{center}
		\includegraphics[width=0.45\textwidth]{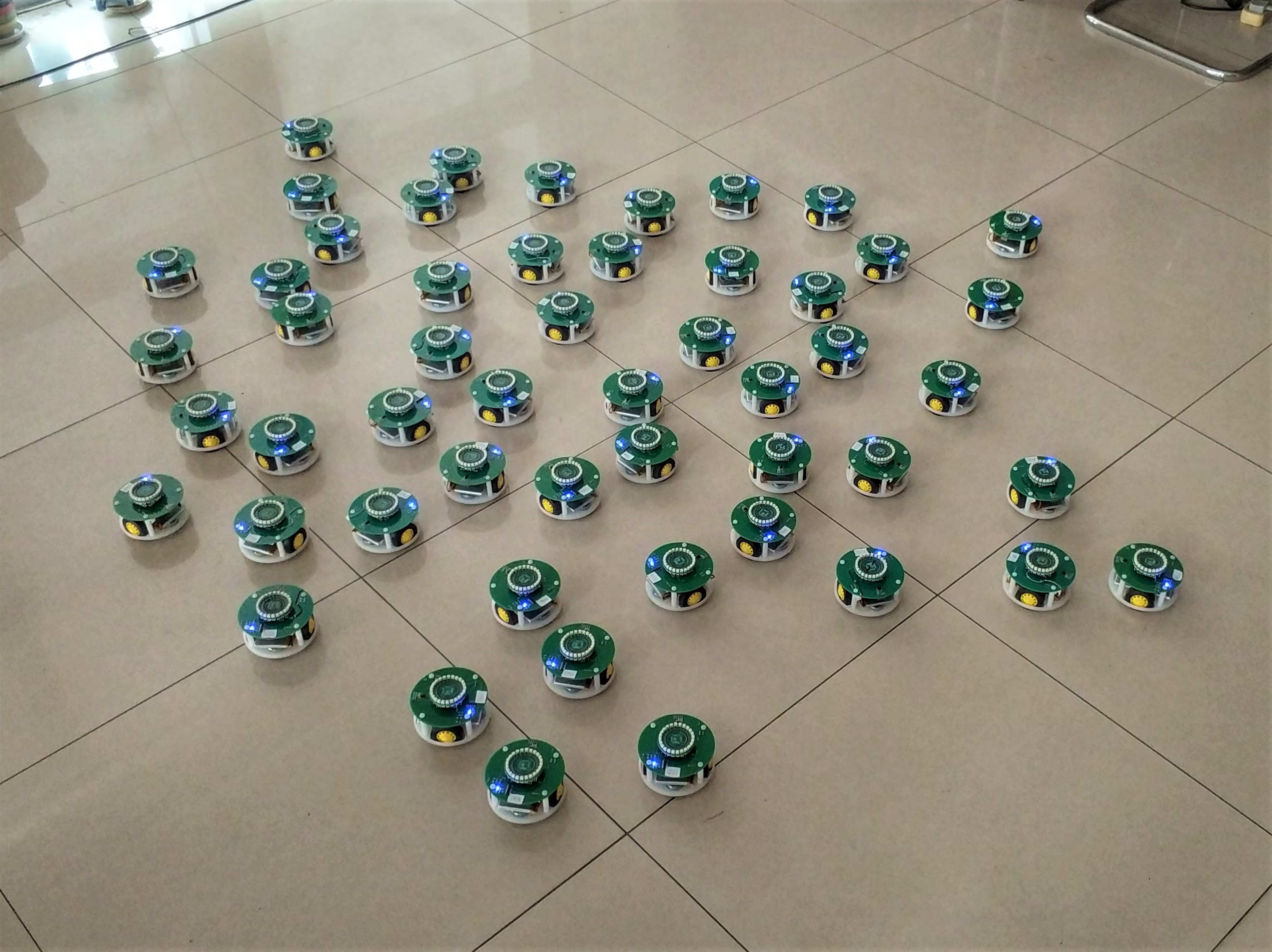}
		\caption{All 52 robots and the test environment.}
		\label{fig:env}
	\end{center}
	\vspace{-6mm}
\end{figure}

\begin{figure}[htbp]
	\begin{center}
		\includegraphics[width=0.45\textwidth]{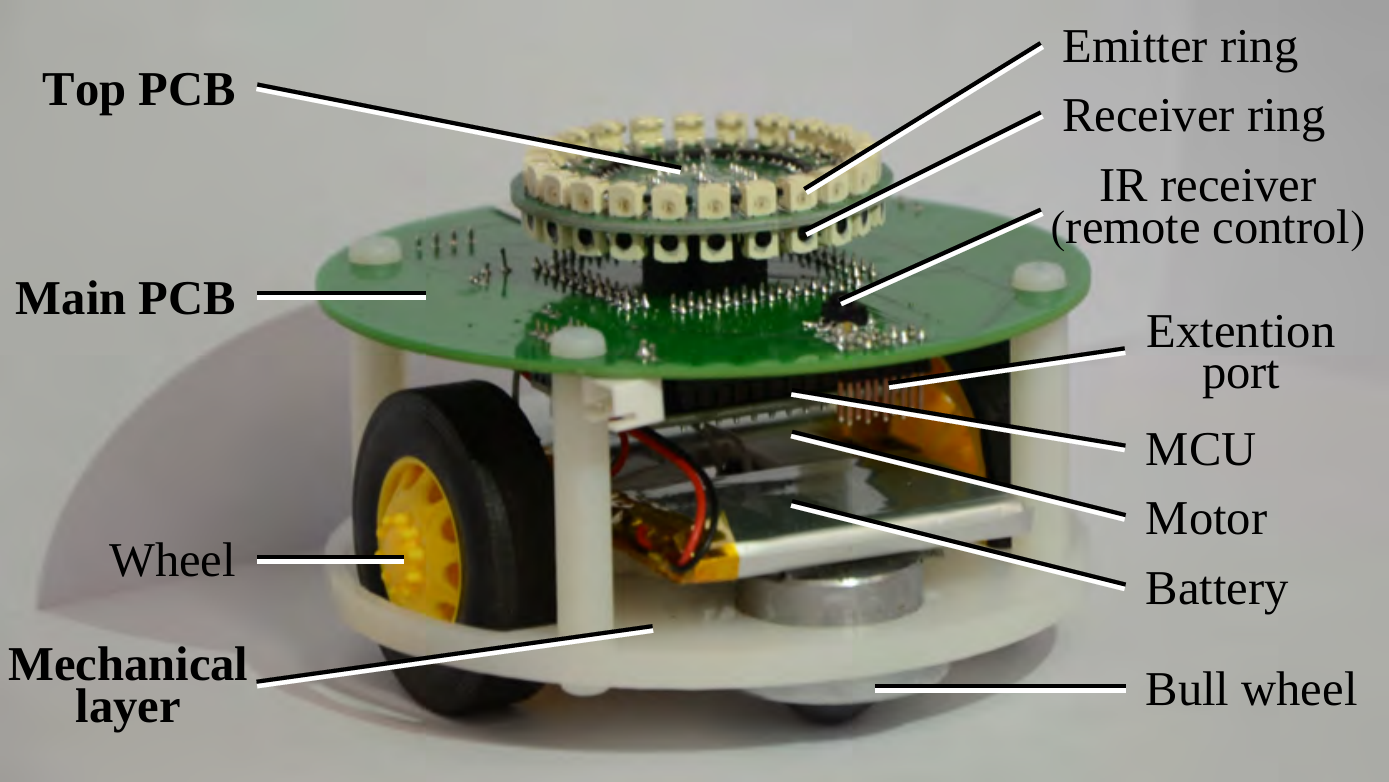}
		\caption{An individual robot. Emitters and receivers are mounted on the top PCB and all other electrical devices are directly mounted on main PCB or connected to main PCB through wires.
		}
		\vspace{-7mm}
		\label{fig:indivrobot}
	\end{center}
\end{figure}

The size of each robot is $10\times10\times7\,\mathrm{cm}$. The robot is controlled by a STM32F407VET6 microprocessor, driven by two active wheels and supported by two bull wheels. QMC5883L three-axes magnetometer is employed to measure the environment magnetic field and provide an agreed direction amongst all robots. An isotropic infrared receiver and an anisotropic infrared emitter consisting of a ring of phototransistor and a ring of infrared LEDs respectively are mounted on the top PCB which has a diameter of $7\,\mathrm{cm}$. Receivers and emitters mounted on top PCB are responsible for VP transmission and calibration.



Using robots equipped with VPE localization algorithm, several localization tests are carried out.

In the first experiment, localization is performed when $N$ robots (N ranging from 2 to 7 are tested in experiment) are distributed evenly on a line with gap of 17.5 cm with $\mathrm{k}_0=0.02,\,\mathrm{k}=0.15$, and $\mathrm{r}_0$ is optimized afterwards to minimize localization error in the equilibrium state. In order to minimize interference, calibration process as mentioned in Remark \ref{rem:calib} is not involved, instead the drift in origin of localization result is eliminated in data processing afterwards by subtracting the average of localization result (for simplicity, in further discussion we name this as `translated localization result'). We consider the VP distribution has reached equilibrium when the variation of translated localization result is less than $0.1 \mathrm{r}_0$ in 200 iterations. In this experiment, each iteration takes approximately 200 ms, and most of the time are consumed in setting up DACs. We estimates that with better implementation each iteration could be reduced to less than 1 ms. Localization error in the equilibrium state and iterations required to reach equilibrium are investigated, results are shown in Figure \ref{fig:experr} and \ref{fig:expline}.

\begin{figure}[htbp]
	\centering
	\includegraphics[width=.4\textwidth]{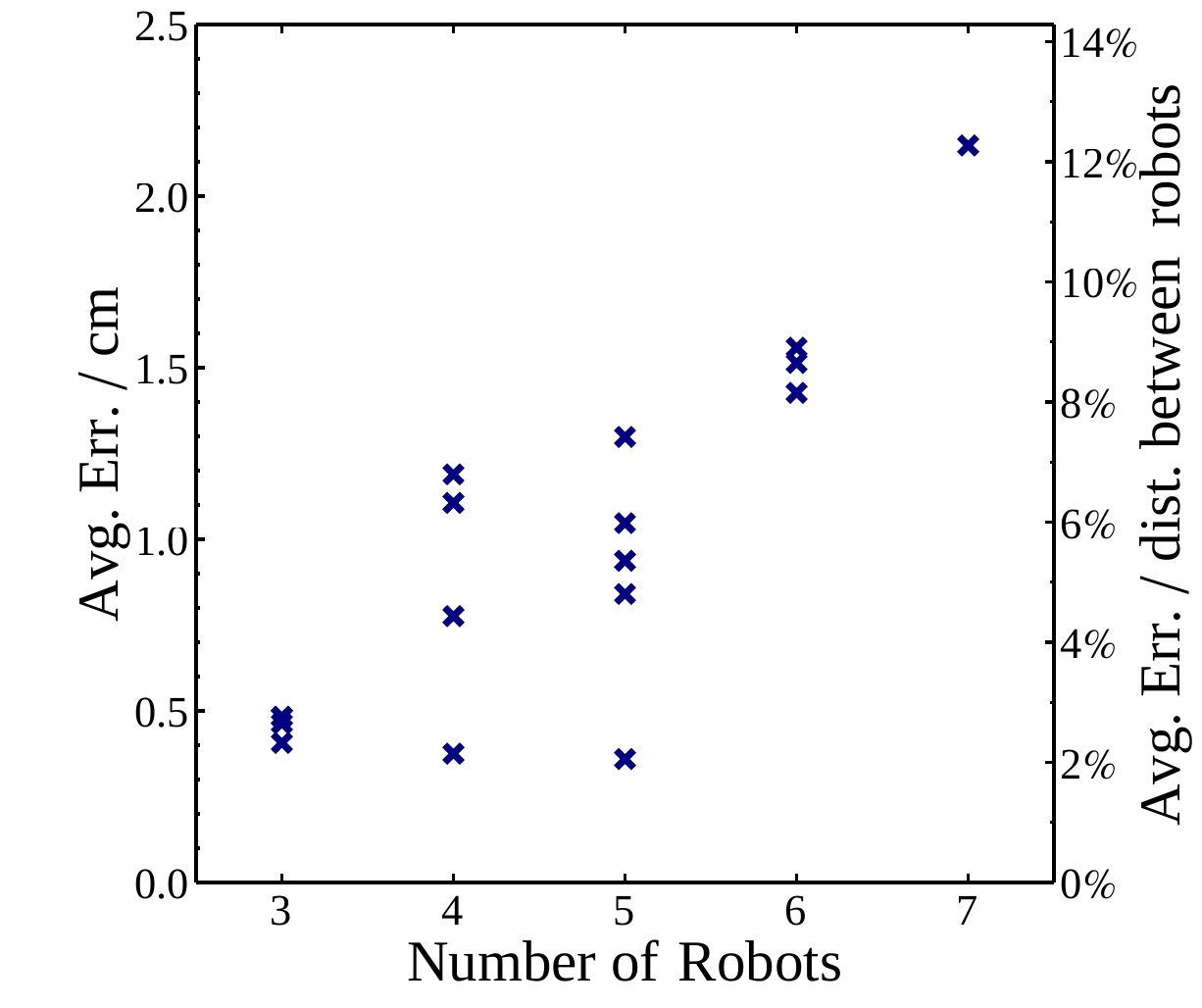}
	\vspace{-2mm}
	\caption{Localization error in experiments with 1D robot distribution.}
	\vspace{-2mm}
	\label{fig:experr}
\end{figure}

\begin{figure}[htbp]
	\centering
	\hspace{-5mm}\includegraphics[width=.4\textwidth]{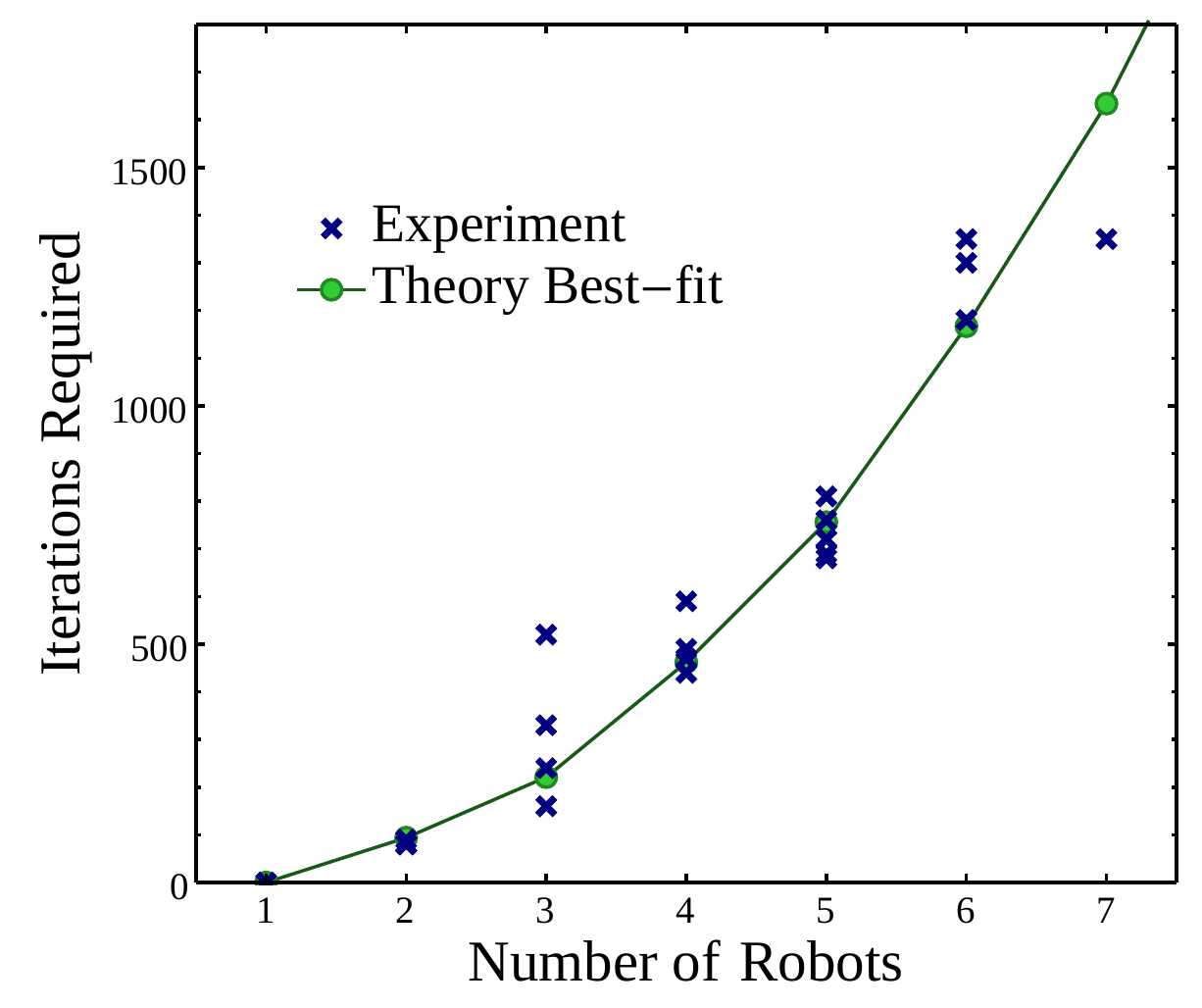}
	\vspace{-2mm}
	\caption{Iterations required to converge in experiments with 1D robot distribution. The green line is the iteration required to reach a average accuracy of 0.1 with maximum light transmission distance of 20 cm and $\mathrm{k}_1=0.0085,\,\mathrm{k}=0.15,\,\mathrm{r}_0=1$.}
	\vspace{-2mm}
	\label{fig:expline}
\end{figure}

As shown in Figure \ref{fig:experr}, in all experiments where equilibrium state is reached, error in translated localization result can be controlled to less than 2.5 cm, which is 15\% of the distance between adjacent robots, and in majority of them, localization error is less than 1.5 cm. Furthermore, as shown in Figure \ref{fig:expline}, localization is completed in less than 6 minutes in all experiments (could be less than 1.5 seconds with better implementation), and experimental points fit well with the theory best-fit.

As mentioned before, the origin of localization result will drift due to the influence of environment and imprecision of sensors. Therefore, we implement the calibration algorithm mentioned in Remark \ref{rem:calib} and execute it every 20 iterations. Experimental results with the same robot configuration with and without calibration algorithm are given in Figure \ref{fig:calib}. While the localization result without calibration keeps drifting in $x+$ direction and reaches a maximum of 35 cm at iteration 2000, with the calibration algorithm implemented, the localization result becomes stable in the end with maximum drift of 1.4 cm throughout the whole localization procedure. 

\begin{figure}[htbp]
	\begin{center}
		\centering
		\setcounter{subfigure}{0}
		
		\subfigure[Without Calibration]{
			\begin{minipage}[t]{0.45\linewidth}
				\centering
				\includegraphics[width=\linewidth]{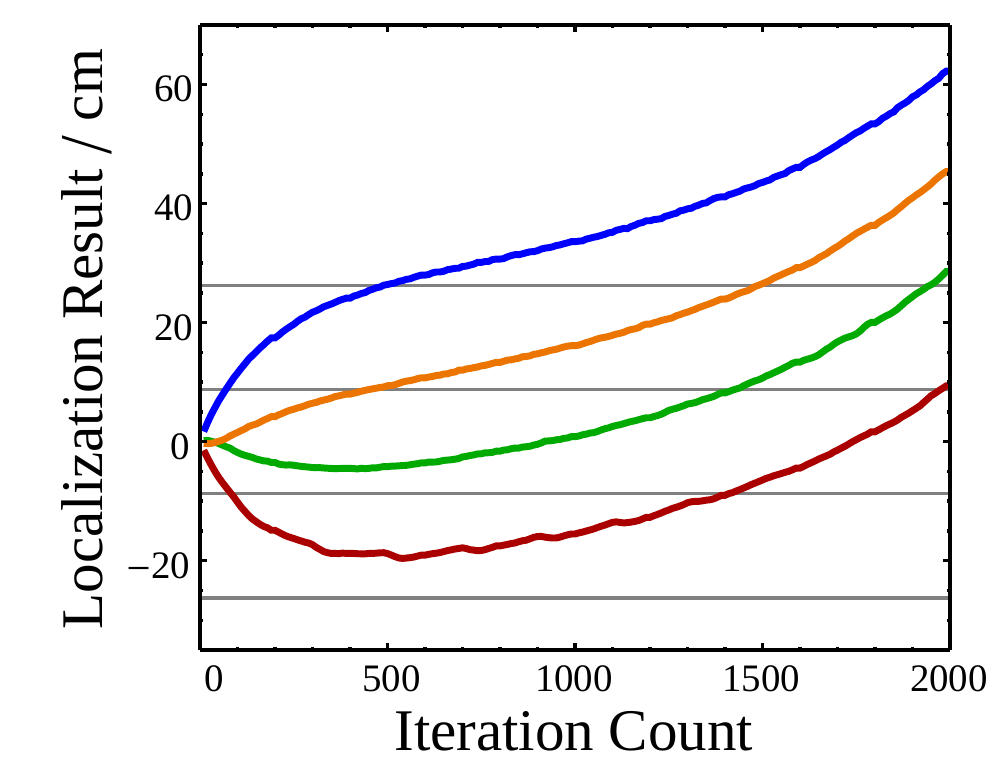}
			\end{minipage}
			\label{fig:withoutcalib}
		}
		\subfigure[With Calibration]{
			\begin{minipage}[t]{0.45\linewidth}
				\centering
				\includegraphics[width=\linewidth]{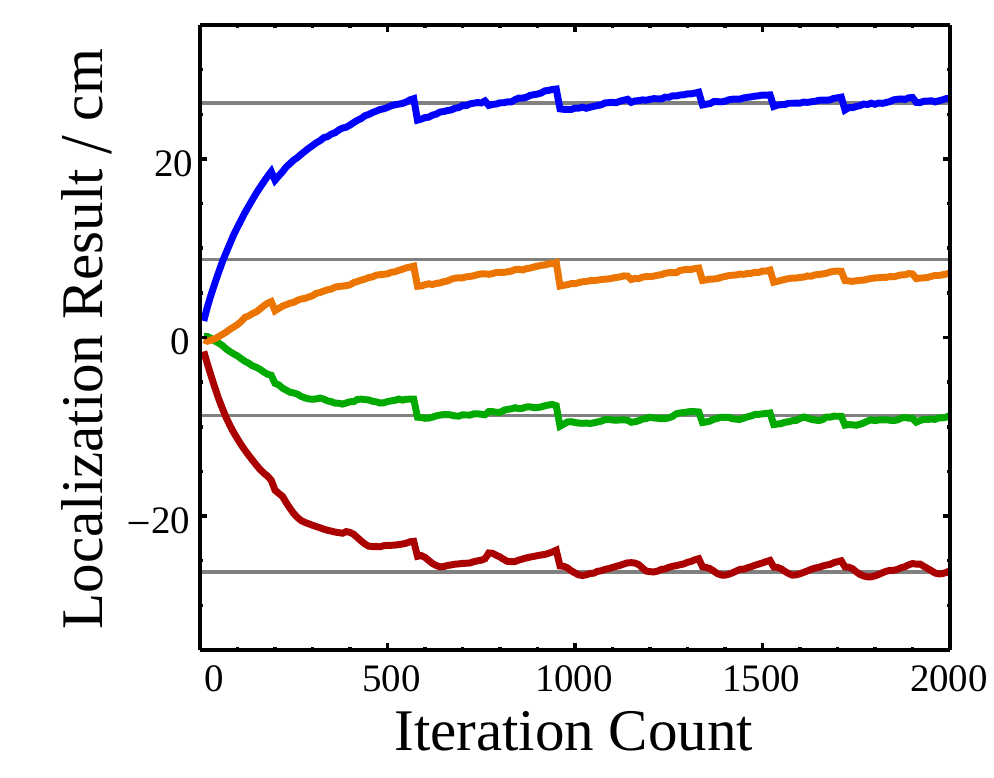}
			\end{minipage}
			\label{fig:withcalib}
		}
		\caption{Comparison between localization process with and without calibration. Gray horizontal lines represent actual position of robots.}
		\label{fig:calib}
	\end{center}
\end{figure}

Then we carried out experiments with 2D robot distribution using 52 robots with calibration algorithm implemented. Parameters used in localization are $\mathrm{k}_0=0.01,\,\mathrm{k}=0.15$, and $\mathrm{r}_0$ is adjusted to minimize the error. Robot distribution, the process of localization, and the trend of localization error with iteration count are provided in Figure \ref{fig:exp2d}.

\begin{figure*}[bp]
	\begin{center}
		\centering
		\includegraphics[width=\linewidth]{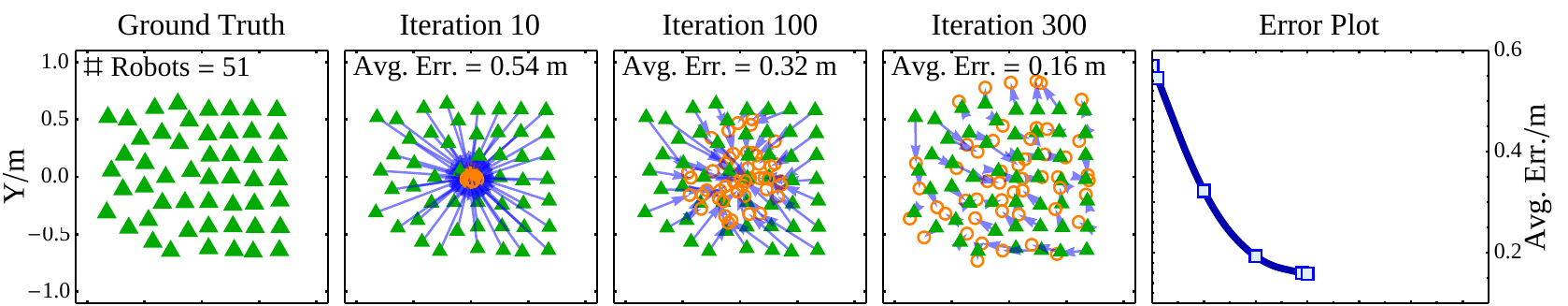}
		\includegraphics[width=\linewidth]{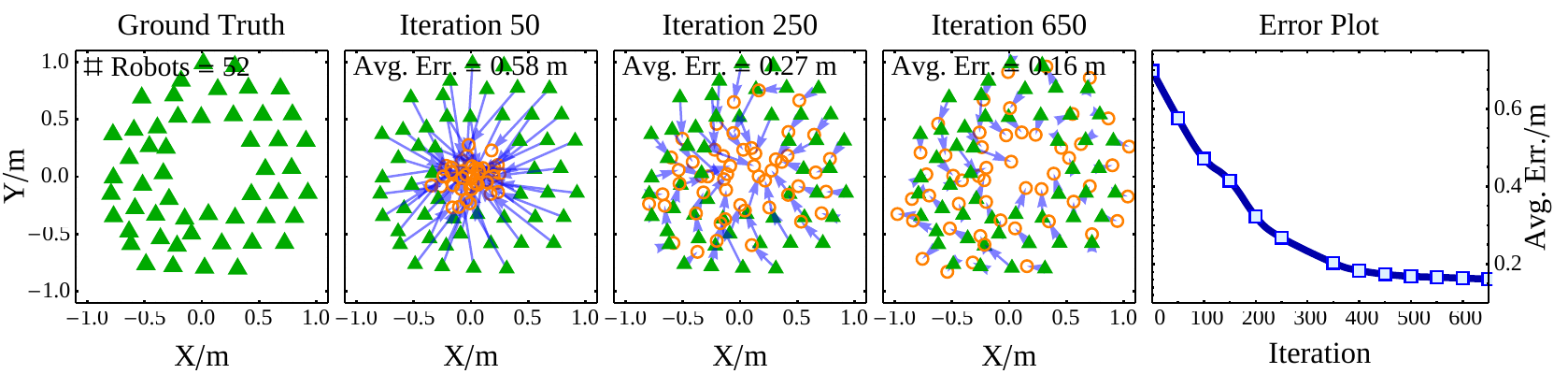}
		\vspace{-6mm}
		\caption{Localization results in experiment with 2D robot distribution. In all figures, both the robots' centroid and the localization result's centroid are set to \{0,0\} through translation, however, in reality the origin of localization result is only 4 cm apart from the actual centroid of the swarm. The first experiment is terminated early because equilibrium is reached in 300 iterations. In the second experiment, we cannot obtain a specific robot's localization result due to hardware issues, thus we neglect this robot in the error analysis.}
		\label{fig:exp2d}
	\end{center}
\end{figure*}

\noindent The average error is approximately half the average distance between adjacent robots. Though not as good as what simulations suggest due to strong environmental influence, it is already sufficient for a variety of tasks including shape formation.

\section{Applications to Shape formation}\label{sec:sfresult}

Based on VPE localization algorithm and combined with a deformation algorithm, autonomous and distributed shape-formation can be achieved.

The general approach used in deformation is similar to Virtual Force Field Method proposed by \cite{hou2012dynamic}, in which virtual forces control the movement of robots. Modification has been made to the collision avoidance algorithm and movement algorithm to better fit our robot system. In our modified algorithm, robots' movement at each step is the summation of two factors: an attractive factor to pull robots inside the target shape and a repulsive factor to keep robots approximately evenly spaced while preventing them from separating from the swarm. The first factor is determined by the position of the robot as well as the targeted shape, and the second factor is determined by the relative position of robots in the neighborhood. The mathematical expression of the displacement of robots in each step is given by

\vspace{-2mm}
\begin{equation}
	\vec{\bm{d}}_{i}=\vec{\bm{d}}_{i,att}(\Omega,\bm{\chi}_i)+\vec{\bm{d}}_{i,rep}(\vec{\bm{r}}_{i,j_n})
	\label{eqn:sfcomp}
\end{equation}
\vspace{-5mm}

\noindent where $\vec{\bm{d}}_{i}$ is the movement of robot $i$ in this step. $\vec{\bm{d}}_{i,att}$ and $\vec{\bm{d}}_{i,rep}$ are the contribution of attractive and repulsive factor. $\Omega$ is the target shape, $\bm{\chi}_i$ is the current localization result of robot $i$ and $\Delta \vec{\bm{r}}_{i,j_n}$ represents the relative displacement of all robots in the neighborhood of robot $i$.

Detailed pseudo-code of the shape formation algorithm used in our simulations and experiments is described in Appendix \ref{sec:collavoid} and \ref{sec:overallalg}.

\subsection{Results in simulations}

In simulations, up to 52 robots are used to form triangle shape and `K' shape. For better comparison with experimental result, average distance between adjacent robots is set to approximately $0.2$ and maximum transmission distance of light is set to $0.5$. Parameters used in localization process are: $\mathrm{k_1}=0.01,\,\mathrm{k}=0.15,\,\mathrm{r}_0=0.35$. Initial localization process takes 100 iterations and preceding re-localization takes 10 iterations each. Results of simulations are shown in Figure \ref{fig:shapesim}.

\begin{figure*}[htbp]
	\centering
    \includegraphics[width=\linewidth]{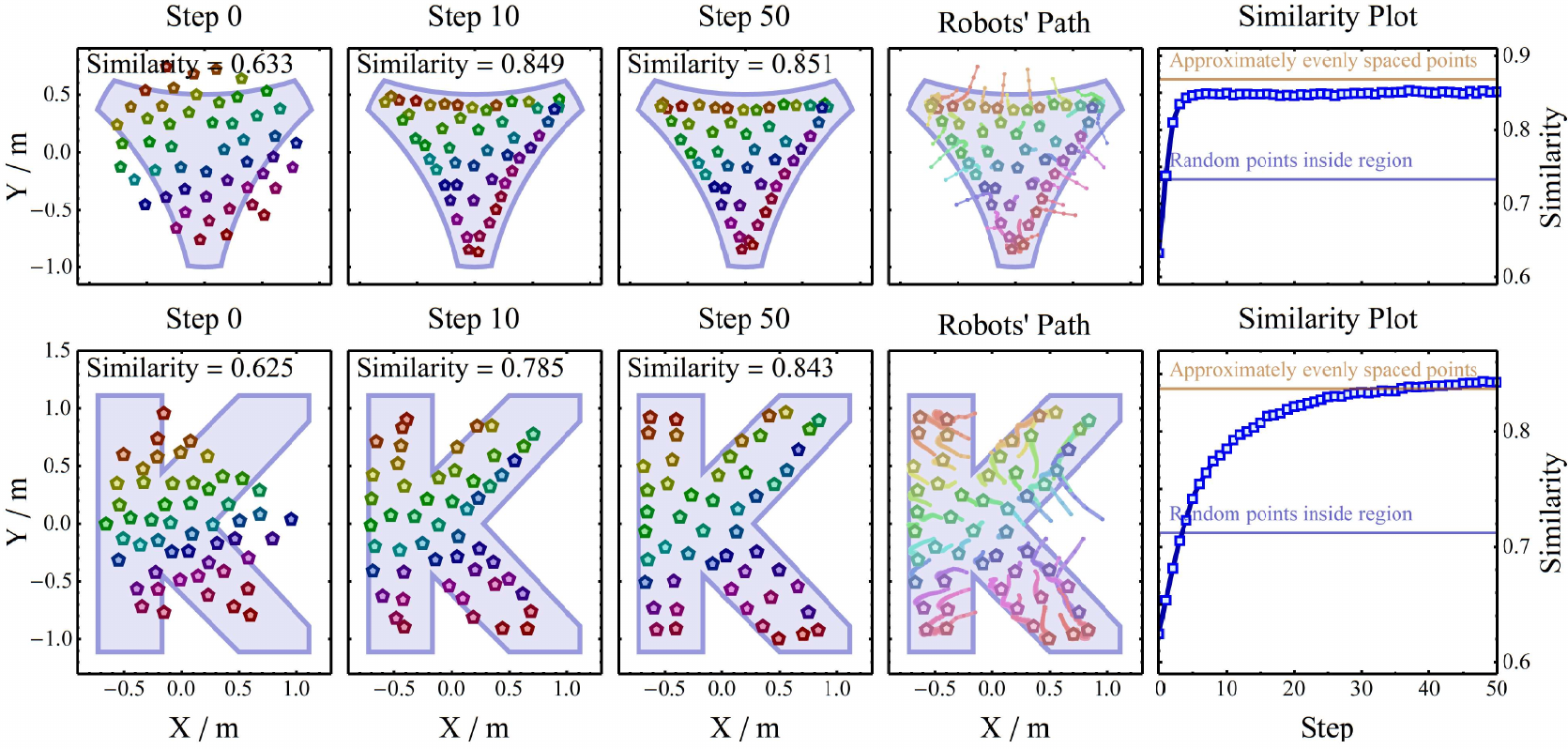}
	\centering
	\caption{Shape formation results in simulations.}
	\label{fig:shapesim}
\end{figure*}


One can observe that, once the program starts, robots are first attracted by the boundary of the shape and will try to fit inside the shape. Notice that in this period, the density of robot is largely different in different parts of the swarm. As time goes, the repulsive factor starts to play its role, equalizing the density and finalizing the shape through small adjustments. During the shape formation process, VPE localization algorithm provides accurate enough localization results with average error at approximately 0.05 unit.

To further develop the shape formation topic, a self-invented method to quantitatively assess the quality of formed shape is presented. The quality of shape formation is evaluated by a similarity value ranging from 0 to 1 which is described by

\vspace{-3mm}
\begin{equation}
	\mathcal{S}\hspace{-.5mm}=\hspace{-1mm}1\hspace{-.5mm}-\hspace{-.5mm}\min\limits_{\sigma}\frac{\int\limits_{\Sigma}\left|in(\vec{\bm{r}},\Omega)\hspace{-.5mm}-\hspace{-.5mm}\frac{A}{\pi\sigma^2l}\sum\limits_{i=1}^{l}e^{-\frac{|\vec{\bm{r}}-\vec{\bm{r}}_i|^2}{\sigma^2}}\right|\mathrm{d}s}{2A}
	\label{eqn:sigma}
\end{equation}
\vspace{-2mm}

\noindent where $A$ is the area of the target shape $\Omega$, $\Sigma$ is the entire plane, $\vec{\bm{r}}_i$ is the position of robot $i$ and $in(\vec{\bm{r}},\Omega)$ is a function defined as

\vspace{-2mm}
\begin{equation}
	in(\vec{\bm{r}},\Omega)=\left\{
	\begin{aligned}[c]
	1& &(\vec{\bm{r}}\not\in\Omega)\\
	0& &(\vec{\bm{r}}\in\Omega)
	\end{aligned}
	\right.
\end{equation}
\vspace{-2mm}

\noindent It is evident that when robots are distributed approximately evenly inside the shape, similarity will be close to 1, and if all robots are far away from target shape then similarity will be close to 0.

The evolution of similarity in a shape formation process is provided in Figure \ref{fig:shapesim}. As expected, similarity rapidly increases with time and eventually reaches a relatively large value, which indicates that the shape formation results resemble the target shapes well in both cases.

\begin{figure*}[htbp]
	\centering
	
    \includegraphics[width=\linewidth]{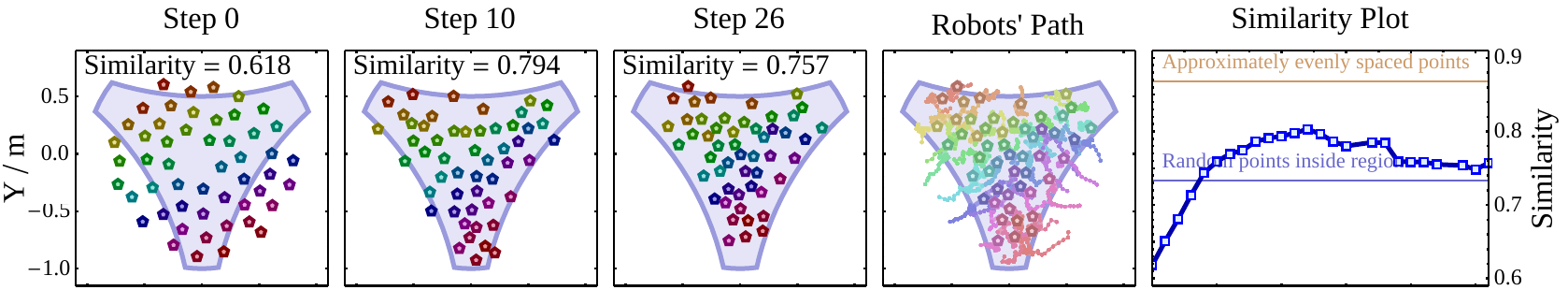}
    \vspace{1.5mm}
    \includegraphics[width=\linewidth]{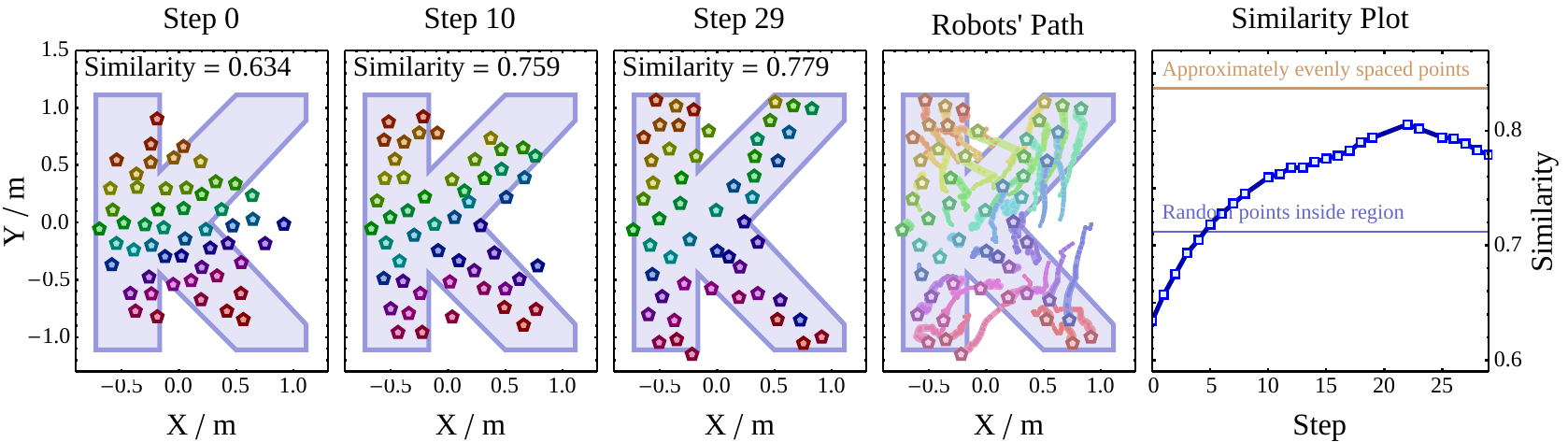}

	\centering
	\vspace{-3mm}
	\caption{Shape formation results in experiments. The target shape has been scaled to fit the produced shape.}
	\vspace{-5mm}
	\label{fig:shapeexp}
\end{figure*}

\subsection{Results in experiments}

In experiments, triangle shape and `K' shape are chosen as targets to test shape formation. In all experiments, parameters of localization are: $\mathrm{k_0}=0.01,\mathrm{k}=0.15, \mathrm{r}_0=0.81\,m$. Initial localization takes 1000 iterations and preceding re-localization takes 100 iterations each. The results are displayed in Figure \ref{fig:shapeexp}.


The robot swarm generates the pre-defined shape in both situations with maximum similarity value of approximately 0.8, which is close to the similarity of approximately evenly spaced points inside the region. The overall trend of robots' paths is consistent with what simulation suggests. The reasons for the decreased similarity value after certain steps are mainly due to the characteristics of similarity evaluation algorithm which prefers points located inside the shape instead of points on the edge of the shape. From the video one can see that subjectively robots are perfecting the shape generated in each step.

Compared to the method proposed by \cite{rubenstein2014programmable} where robots line up and move on the boundary of the shape, our approach to deformation is more fluid like, which requires a magnitude less iterations in deformation and resembles the behavior of natural swarm systems better. Rapid and accurate shape formation achieved in both simulations and experiments shows that VPE localization algorithm is suitable for providing relative localization results when executing more complex tasks.

\section{Discussion}\label{sec:discussion}

Though with multiple merits, VPE localization algorithms has its own limitations which requires further optimization. As proved in Section \ref{sec:anaVPEmethod}, VPE localization algorithm has $O(l)$ asymptotic time complexity, i.e. it can achieve a certain accuracy $\delta_0$ in at most $\mathrm{\alpha} l$ iterations where $\alpha$ is a constant. However, coefficient $\alpha$ is at the magnitude of several hundreds in simulations and experiments, which is far from theoretical limit of $0.5$. This phenomenon can be observed in Figure \ref{fig:convt} and \ref{fig:expline} where iterations required to converge do not increase linearly with the size factor when size factor is smaller than 10. Thus when operating on swarms with less than several hundreds of robots, the algorithm could even be slower compared with algorithms which have higher asymptotic time complexity. This problem can be partially addressed by implementing acceleration algorithms to speed up converging process. A simple example for such acceleration algorithm is setting larger $\mathrm{k}_1$ and $\mathrm{k}_2$ when the localization result is already near equilibrium. In authors' crude simulations, such method can provides up to 90\% acceleration at a slight risk of failing to converge. Furthermore, in tasks like shape formation, after initial localization, subsequent localization will be much faster as well. If VPE localization algorithm is used in systems without direct communication ability using Algorithm \ref{alg:2}, the requirement for time synchronization between all robots in the swarm can be tricky as well. Currently, authors implement a simple algorithm to achieve time synchronization, in which a robot would start emitting light whenever its timer is up or it sensed light emitted by other robots. Such method might be prone to failure with larger swarms, but fortunately, VPE localization algorithm is resistant to addition and removal of robots, a robot can just skip an iteration if it fails to synchronize with other robots without influencing the overall localization process. Another property of VPE localization algorithm is that it relies more on the interaction between robots than internal calculation of each robot, so while the cost for internal computation and memory decreased, the cost for interaction increased. Because emitting signals are usually more energy consuming than internal computation, hardware and software designed for VPE localization algorithm needs to be as energy-efficient as possible.

We designed the robot just as a proof of concept of VPE localization algorithm, so the hardware and software are all designed to be as simple as possible, and performance and stability are not our major concern. For example, light sensor and emitters are chosen by convenience of assembly instead of response speed or accuracy, and codes are not optimized for error cancelling or energy saving. The crude design partially accounts for the slow speed and undesirable error of localization, which can possibly be fixed by following measures:

\begin{itemize}
	\item Use frequency modulated signal instead of DC signal to reduce the influence of environment and extend the range of signal. For example, we can change line \ref{alg:sst} and \ref{alg:sst1} in Algorithm \ref{alg:2} to the following:
	
	\hspace{4mm}Emit signal with frequency $\mathrm{f}_0$ and strength $I_2(\hat{\bm{r}})$.

	\hspace{4mm}Emit signal with frequency $\mathrm{f}_0+\mathrm{f}_1 \lfloor\log_{10}(\xi_+)\rfloor$ and strength $I_1(\hat{\bm{r}})/10^{\lfloor\log_{10}(\xi_+)\rfloor}$ where $\mathrm{f}_1$ is another frequency constant.

	In this way, most environmental noise can be filtered out by a band-pass filter. Furthermore, currently the size of swarm our hardware can support is restricted because the maximum of $I_1(\hat{\bm{r}})$ increases with the size of the swarm and emitters will be saturated if the swarm is too large in size. But by converting strength to frequency, we can avoid saturation of emitters thus enabling implementation in large scale swarms.
	
	\item Use high frequency LED to increase the response of emitter and use PIN photodiode or negatively biased PN photodiode to increase the response of receiver.
\end{itemize}

\section{Conclusion}\label{sec:conclusion}

This paper presents VPE localization algorithm for localizing inside homogeneous robot swarms. In the algorithm, robots repetitively exchange VPs with neighboring robots and obtain their location results based on the amount of VPs they own in the final state.

The convergence of the algorithm is verified by first reducing the problem to a convergence problem of $\lim_{n\to\infty} \bm{T}^{n}\,\bm{\xi}$ and then applying Perron–Frobenius theorem. Furthermore, a specific scenario where robots are evenly distributed on a line is investigated to evaluate the asymptotic time complexity of the algorithm, which shows that VPE localization algorithm has asymptotic time complexity of $\Theta(l)$ where $l$ is span of the swarm on $x$ direction.

VPE localization algorithm is then evaluated in simulations and robot swarm experiments to investigate whether VPE localization algorithm is suitable for real applications with noise and interference. Results confirm that our algorithm can be applied in real world robot swarms by achieving localization result with error at approximately half of the average distance between robots, and forming accurate triangle shape and 'K' shape in shape formation tasks.

%
%

\section*{Acknowledgment}
Special Thanks to Yi Zhao for discussing with us when the project just started and providing assistance with finding proper localization algorithm.
\bibliographystyle{IEEEtran}
\bibliography{distributedlocalization}
\appendix

\subsection{Calibration algorithm for total VP amount}\label{sec:calib}

\vspace{3mm}

\begin{breakablealgorithm}\label{alg:calib}
	\caption{Pseudo-code for calibration of total VP amount}
	\begin{algorithmic}[1]
		\Require All robots:
		\begin{enumerate}
			\item [a.] Can emit light with isotropic angular distribution and sense ambient light intensity.
			\item [b.] Knows the VP amount it owns (named as $\xi$).
		\end{enumerate}
		\Ensure Normalized VP amount $\xi'$ for all robots which satisfies $\xi_i/\xi_j=\xi'_i/\xi'_j$ and $\sum_{i}\xi'_i=l$.
		\Notation
		\Notationx{$\xi_{A}$ - floating point variable used in the calibration process.}
		\Notationx{$n_{m}$ - iterations to calculate (predefined).}
		\Notationx{$\mathrm{k}_3, \mathrm{k}_4$ - constants controlling the intensity of light emitted.}
		\ForEach{robot $i\in \bm{V}$.}
		\State $\xi_{A}\gets \xi$
		
		\MyState{Begin emitting isotropic light signal with intensity of $\mathrm{k}_4$ unit.}\label{alg:sstcal}
		
		\State $c\gets$ sensed ambient light intensity.
		\State Stop emission.\label{alg:sedcal}
		
		\For{$n=0\to n_m-1$}
			\State Time synchronization.
			\State Possible addition: \textbf{repeat} line \ref{alg:sstcal}-\ref{alg:sedcal}.
			\MyState{Begin emitting isotropic light signal with intensity of $\mathrm{k}_3\,\xi_{A}$ unit.}		
			\State $s\gets$ sensed ambient light intensity.
		    \State Stop emission.
			\State $\xi_{A}\gets\left( 1 - \frac{c k_3}{k_4} \right) \xi_{A} + s$
		\EndFor
		
		\State $\xi'\gets\frac{\xi}{\xi_{A}}$
	\end{algorithmic}
\end{breakablealgorithm}
\vspace{3mm}

One might have noticed that this is just Algorithm \ref{alg:2} with $k=0$, thus it is not hard to understand why this algorithm can serve its purpose.

\subsection{Modified collision avoidance algorithm}\label{sec:collavoid}
By exploiting the fact that a robot would sense higher intensity of light when placed closer to other robots, an algorithm to keep robots approximately evenly spaced is created in order to avoid collision between robots. Detailed pseudo-code for collision avoidance algorithm is shown in Algorithm \ref{alg:4}.

\vspace{3mm}
\begin{breakablealgorithm}\label{alg:4}
	\caption{Pseudo-code for collision avoidance}
	\begin{algorithmic}[1]
		\Require All robots:
		\begin{enumerate}
			\item [a.] Share a common $x+$ direction.
			\item [b.] Can emit light with anisotropic angular distribution and sense ambient light intensity.
		\end{enumerate}
		\Ensure repulsive factor $\vec{\bm{d}}_{i,rep}$ in (\ref{eqn:sfcomp}) for all robots.
		\Notation
		\Notationx{$\hat{\bm{r}}$ - light emitting direction.}
		\Notationx{$\mathrm{k}'$ - constant controlling the anisotropic distribution of light emitted.}
		\Notationx{$\mathrm{C}_1$,$\mathrm{C}_2$ - constants controlling the strength of repulsive factor.}
		\Notationx{$I_0$ - constant controlling the desired distance between robots.}
		\Notationx{$d_0, d_1$ - constants controlling the strength of repulsive factor.}
		\ForEach{robot $i\in \bm{V}$.}
		
		\MyState{Begin emitting light with angular distribution	
			\begin{equation}
			I(\hat{\bm{r}})=e^{\mathrm{k}' \hat{\bm{r}}\cdot\hat{\bm{x}}}\hspace{5mm}(\mathrm{k_3}>0)
			\end{equation}
			\label{alg:s2}}
		
		\State $I_{x,+}\gets$ sensed ambient light intensity.
		
		\State Begin emitting light with angular distribution $I(-\hat{\bm{r}})$.
		
		\State $I_{x,-}\gets$ sensed ambient light intensity.
		
		\MyState{$d_{rep,x}\hspace{-.7mm}\gets\hspace{-.7mm}\mathrm{C}_1 \tanh \left(\mathrm{C}_2 \left(\max (I_{x,+},I_{x,-})\hspace{-.7mm}-\hspace{-.7mm}I_0\right)\right)\hspace{-.5mm}\frac{I_{x,+}-I_{x,-}}{I_{x,+}-I_{x,-}}$ where $d_{rep,x}$ represents the $x$ component of $\vec{\bm{d}}_{rep}$\label{alg:s6}}
		
		\MyState{\textbf{repeat} line \ref{alg:s2}-\ref{alg:s6} for $y$ direction and obtain $y$ component of repulsive factor $d_{rep,y}$.}
		
		\State $\vec{\bm{d}}_{rep}\gets\{d_{rep,x},d_{rep,y}\}$
	\end{algorithmic}
\end{breakablealgorithm}

\vspace{3mm}
When a robot is far away from the main swarm, the intensity of light it sensed would become weak, such that the repulsive factor of displacement will point to the direction where higher intensity is received, which is the direction towards the main swarm. On the other hand, if a robot is too close to another robot, a large repulsive factor with direction opposite to the other robot will be generated. This method can yield fair collision avoidance and aggregation result, yet robots do not need to have the ability to detect neighbors or know the relative distance and angle of neighboring robots.

\subsection{Overall shape formation algorithm} \label{sec:overallalg}

The overall shape formation algorithm is presented in Algorithm \ref{alg:5}.\\

\begin{breakablealgorithm}\label{alg:5}
	\caption{Modified Virtual Force Field Algorithm}
	\begin{algorithmic}[1]
		\Require All robots:
		\begin{enumerate}
			\item [a.] Know the target shape $\Omega$.
			\item [b.] Can calculate its position relative to the swarm $\bm{\chi}$.
			\item [c.] Can calculate the repulsive factor $\vec{\bm{d}}_{rep}$ given by Algorithm \ref{alg:4}.
		\end{enumerate}
		\Ensure Achieve shape formation.
		\Notation
		\Notationx{$\mathrm{C}_3$,$\mathrm{C}_4$ - constants controlling the strength of attractive factor.}
		\ForEach{robot $i\in \bm{V}$.}
		
		\While{\textbf{true}}
			\State Calculate position relative to the swarm $\bm{\chi}$.
			\MyState{$\vec{\bm{p}}_{near}\gets \mathop{\mathrm{argmin}}_{\vec{\bm{p}}\in\partial\Omega} |\bm{\chi}-\vec{\bm{p}}|$ where $\vec{\bm{p}}_{near}$ represents the closest point to $\vec{\bm{p}}$ on the boundary of region $\Omega$.}
			\State $\vec{\bm{f}}\gets\vec{\bm{p}}_{near}-\bm{\chi}$
			\MyState{$\vec{\bm{d}}_{att}\gets\mathrm{C}_3 (1+\tanh(\mathrm{C}_4 (2 in(\bm{\chi},\Omega)-1)|\vec{\bm{f}}|)\,(2 in(\bm{\chi},\Omega)-1)\hat{\bm{f}}$}
			\State Calculate $\vec{\bm{d}}_{rep}$.
			\State Move the robot by $\vec{\bm{d}}=\vec{\bm{d}}_{att}+\vec{\bm{d}}_{rep}$.
		\EndWhile
	\end{algorithmic}
\end{breakablealgorithm}
\vspace{3mm}

\end{document}